\documentclass{article}

\PassOptionsToPackage{numbers, compress}{natbib}

\usepackage[preprint]{neurips_2026}

\usepackage[utf8]{inputenc} 
\usepackage[T1]{fontenc}    
\usepackage{hyperref}       
\usepackage{url}            
\usepackage{booktabs}       
\usepackage{amsfonts}       
\usepackage{nicefrac}       
\usepackage{microtype}      
\usepackage{xcolor}         

\usepackage{amsmath, amssymb, amsthm}
\usepackage{bbm}
\usepackage{math_command}

\usepackage{algorithm}
\usepackage{algpseudocode}
\usepackage{amsmath,amssymb,amsfonts}
\usepackage{amsthm}
\usepackage{hyperref}
\usepackage{cleveref}
\usepackage{enumitem}
\setlist{nosep}

\definecolor{darkblue}{rgb}{0, 0, 0.5}
\hypersetup{colorlinks=true, citecolor=darkblue, linkcolor=darkblue, urlcolor=darkblue}

\newtheorem{theorem}{Theorem}[section]

\newtheorem{lemma}[theorem]{Lemma}

\newtheorem{definition}[theorem]{Definition}
\theoremstyle{remark}

\usepackage{booktabs}
\usepackage{multirow}

\usepackage{xcolor}
\usepackage{colortbl}
\definecolor{baselinegrey}{RGB}{235, 235, 235}
\definecolor{winnerpink}{RGB}{255, 220, 230}
\definecolor{loserblue}{RGB}{220, 230, 250}
\title{A Semantic-Sampling Framework for Evaluating Calibration in Open-Ended Question Answering}

\author{%
Zhanliang Wang$^{1}$\thanks{Equal contribution.}\ , Jiancong Xiao$^{1*}$, Ruochen Jin$^2$, Shu Yang$^1$, Bojian Hou$^{1}$\thanks{Corresponding authors.}\ , and  Li Shen$^{1\dagger}$\\
$^1$University of Pennsylvania, Philadelphia, PA; $^2$Dartmouth College, Hanover, NH\\
\texttt{\{aaronwzl,jcxiao\}@upenn.edu, ruochen.jin.gr@dartmouth.edu,}\\ \texttt{\{syang11,bojianh,lishen\}@upenn.edu}}

\begin{document}


\maketitle

\begin{abstract}
Calibration measures whether a model's predicted confidence aligns with its
empirical accuracy, and is central to the reliable deployment of large
language models (LLMs) in high-stakes domains such as medicine and law.
While much recent work focuses on \emph{improving} LLM calibration, the
equally important question of how to \emph{evaluate} it in realistic
settings remains underdeveloped. Open-ended question answering (QA), the
most common deployment setting for modern LLMs, is where existing
evaluation methods fall short: logit-based metrics need restricted output
formats and internal probabilities; verbalized confidence is self-reported
and often overconfident; and sampling-based methods rely on task-specific
extraction rules without a clear finite-sample target. We introduce
\textbf{Sem-ECE} (\textbf{Sem}antic-Sampling \textbf{E}xpected
\textbf{C}alibration \textbf{E}rror), a calibration evaluation framework
for open-ended QA that samples answers from the model, groups them into
semantic classes, and uses the resulting frequencies as confidence. We
study two estimators within this framework: Sem$_1$-ECE, the same-sample
self-consistency score, and Sem$_2$-ECE, a held-out variant that
separates answer selection from confidence evaluation. We prove both are
asymptotically unbiased, and further show that they agree on easy
questions but diverge on hard ones with Sem$_2$ achieving strictly
smaller calibration error, so their gap also serves as a diagnostic for
question difficulty. Experiments on three open-ended QA benchmarks across
five leading commercial LLMs match our theoretical predictions and show
that Sem-ECE outperforms verbalized confidence and existing
sampling-based methods, while complementing logit-based evaluation when
internal probabilities are unavailable.\footnote{Code is available at \url{https://github.com/ZhanliangAaronWang/Sem-ECE}.}

\end{abstract}

\section{Introduction}
\label{sec:introduction}
Calibration measures whether a model's predicted confidence aligns with its
empirical accuracy, and is widely recognized as a prerequisite for the reliable
deployment of large language models (LLMs)~\citep{guo2017calibration,
kadavath2022language,tian2023just}. In high-stakes domains such as medicine and
law, a system that is accurate on average but poorly
calibrated cannot distinguish routine queries from queries on which it is
likely to fail, leaving downstream pipelines without a signal for
when to trust an answer, abstain, or escalate.

Much recent work focuses on \emph{improving} LLM calibration, through post-hoc
rescaling, prompting strategies, or calibration-aware
fine-tuning~\citep{platt1999probabilistic,guo2017calibration,kull2019beyond,
kumar2019verified,tian2023just,xiao2025restoring}. The equally important
question of how to \emph{evaluate} calibration in realistic settings remains
underdeveloped. Classical metrics such as Brier score, reliability diagrams,
and expected calibration error~\citep{brier1950verification,naeini2015obtaining,
guo2017calibration} fit classification and multiple-choice QA but break down
in open-ended QA, the dominant deployment setting for modern LLMs: the answer
space is unbounded, two answers worded very differently can be equally correct,
and commercial APIs frequently do not expose logits. Existing black-box
approaches each address part of this gap but none covers the full setting with
a statistically explicit target. Verbalized confidence is format-agnostic
\citep{lin2022teaching,kadavath2022language,mielke2022reducing,tian2023just}
but depends on self-reporting and is frequently
overconfident~\citep{kadavath2022language,tian2023just,wei2024simpleqa}.
Sampling-based methods derive confidence from the consistency of repeated
generations~\citep{wang2023selfconsistency,lyu2025calibrating} but typically
require task-specific extraction rules and rely on heuristic frequency scores
rather than rigorous statistical targets.

We introduce \textbf{Sem-ECE} (Semantic-Sampling Expected Calibration Error),
a semantic-sampling framework for calibration evaluation in
open-ended QA. The framework repeatedly samples answers from the model, maps
free-form generations to semantic answer classes via an LLM judge, and
evaluates calibration from the resulting semantic frequencies, without
requiring logits, multiple-choice options, or hand-crafted answer-extraction
rules. Within this framework, we study two natural estimators of the same
target, i.e., the probability of the model's most likely semantic answer.
\textbf{Sem$_1$-ECE} is the standard same-sample self-consistency score: it
selects the most frequent semantic answer and uses that same frequency as
confidence. \textbf{Sem$_2$-ECE} is a held-out variant that selects the
answer on one block of samples and measures its frequency on a disjoint
held-out block. We prove that both are asymptotically unbiased, placing
sampling-based calibration evaluation on a principled statistical footing,
and show in closed form that on hard low-margin questions Sem$_2$ yields a
strictly smaller calibration error than Sem$_1$, while on easy questions the
two are nearly indistinguishable; the Sem$_1$--Sem$_2$ gap thus also serves
as a simple observable diagnostic for question difficulty.

Sem-ECE improves over verbalized confidence by measuring a behavioral
property of the answer distribution rather than relying on self-reporting,
and advances existing sampling-based calibration evaluation by replacing
hand-crafted extraction rules and heuristic frequency scores with estimators
that have an explicit population target and provable guarantees; it
complements logit-based evaluation when internal probabilities are
unavailable. Experiments on three challenging open-ended QA benchmarks,
including Humanity's Last Exam, across five leading commercial LLMs
(ChatGPT, Claude, Gemini, Grok, and Mistral) confirm our theoretical
predictions, with Sem$_2$-ECE achieving lower calibration error than
verbalized confidence on the large majority of model--benchmark pairs.
\section{Related Work}
\label{sec:related}

Calibration evaluation is well-studied for probabilistic classifiers and
multiple-choice QA via Brier score, reliability diagrams, and binned ECE
\citep{brier1950verification,naeini2015obtaining,guo2017calibration}, but
open-ended QA breaks these tools: the answer space is unbounded, correctness
is semantic rather than lexical, and commercial APIs often do not expose
logits. Two families of black-box confidence sources have emerged.
\emph{Verbalized confidence} elicits the model's stated uncertainty in words
or as a probability
\citep{lin2022teaching,kadavath2022language,mielke2022reducing,tian2023just},
but is self-reported and frequently overconfident. \emph{Sampling-based
methods} use agreement across repeated generations as a confidence signal
\citep{wang2023selfconsistency,lyu2025calibrating}, with semantic-uncertainty
variants grouping generations by meaning
\citep{kuhn2023semantic,farquhar2024detecting}; existing instantiations rely
on task-specific answer-extraction rules and lack an explicit population
target. Sem-ECE measures a behavioral property of the answer distribution
like sampling-based methods, but assigns the resulting frequency an explicit
asymptotic target with provable guarantees, distinguishing it from heuristic
frequency scores and self-reported uncertainty. A complementary line of work
aims to \emph{improve} calibration via post-hoc rescaling or fine-tuning
\citep{platt1999probabilistic,kull2019beyond,kumar2019verified,xiao2025restoring};
see \Cref{app:related} for an extended discussion.
\section{Preliminaries}
\label{sec:preliminaries}

\textbf{Semantic answer space and oracle confidence.} Let \(\mathcal Q\) be a distribution over questions. For a fixed
\(q\sim\mathcal Q\), querying the LLM under a fixed prompt and
decoding configuration produces a random free-form answer string. Two
strings are \emph{semantically equivalent} if they express the same
answer to \(q\); the equivalence classes form a finite \emph{semantic
answer space} \(\mathcal Z_q = \{1,\ldots,K_q
\}\), with
\(K_q := |\mathcal Z_q|\). The LLM induces a categorical distribution
\(\pi_q\) on \(\mathcal Z_q\), with
\(\pi_{q,k} := \pi_q(k) = \Pr(\text{the LLM's answer to } q \text{
lies in class } k)\). The \emph{population semantic mode} is
\(z_q^\star := \argmax_{k}\pi_{q,k}\) (ties broken by a fixed
deterministic rule), and the \emph{oracle semantic confidence} is
\(c_q^\star := \pi_{q,z_q^\star} = \max_{k}\pi_{q,k}\). It is an
agreement quantity, not a correctness quantity: a model can have
\(c_q^\star = 1\) and still be wrong on every sample.

\textbf{Semantic correctness.} Correctness is defined at the semantic-class level. Let
\(Y_q : \mathcal Z_q \to \{0,1\}\) be the correctness function for
\(q\), with \(Y_q(k) = 1\) iff class \(k\) is correct relative to the
reference answer. If a method commits to class \(k\) with confidence
\(c \in [0,1]\), its calibration is evaluated using the pair
\((c, Y_q(k))\); calibration is thus assessed at the semantic level
rather than on raw strings.

\textbf{Empirical estimation of $\pi_q$.}
The distribution $\pi_q$ is unknown; we access it through $n+m$
independent generations clustered into semantic classes
$Z_1, \ldots, Z_{n+m} \overset{\mathrm{i.i.d.}}{\sim} \pi_q$, and
partition the index set $[n+m]$ into a selection block $N$ of size
$n$ and a disjoint evaluation block $E$ of size $m$. For any
$I \subseteq [n+m]$, the empirical semantic PMF is
$\hat\pi_I(k) := |I|^{-1}\sum_{i\in I}\mathbf{1}\{Z_i = k\}$,
and the corresponding empirical semantic mode is
$\hat z_I := \arg\max_{k\in\mathcal{Z}_q}\hat\pi_I(k)$
(ties broken by the same deterministic rule as for $z_q^\star$).
We write $\hat z_N$ for the empirical mode on the selection block;
it is the answer the model would deploy.

\textbf{Standardized margin.} Two scalar summaries of \(\pi_q\) will be referenced repeatedly. The
\emph{top-two margin}
\(\Delta_q := \pi_{q,z_q^\star} - \pi_{q,z_q^{(2)}}\)
is the gap between the modal probability and the runner-up; the
\emph{top-two probability mass}
\(p_q := \pi_{q,z_q^\star} + \pi_{q,z_q^{(2)}}\) is
their sum, where \(\pi_{q,z_q^{(2)}}:=\max_{k \ne z_q^\star}\pi_{q,k}\). From these we form the \emph{standardized margin}
\(\tilde m_q := \Delta_q / \sqrt{p_q / n}\) and its half
\(\tilde\lambda_q := \tilde m_q / 2\): \(\tilde m_q\) is the z-score
of \(\Delta_q\) under the leading-order variance \(p_q / n\) of the
differential count
\(\hat\pi_N(z_q^\star) - \max_{k \ne z_q^\star}\hat\pi_N(k)\), and
parametrizes the regime structure throughout \Cref{sec:theory}. 

\textbf{The $p_q\to 1$ convention.}\label{sec:empirical-access}
For figure readability we adopt the convention $p_q \to 1$, under which $\tilde m_q = \sqrt{n}\,\Delta_q$; all theorems are stated for general $p_q \in (0,1]$.

\textbf{Binned expected calibration error.}
\label{sec:calibration-metric}
For a (confidence, correctness) pair $(c, a)$ with $c \in [0,1]$ and
$a \in \{0,1\}$, calibration is measured by the binned expected
calibration error~\citep{naeini2015obtaining,guo2017calibration}. Fix
$L$ equal-width bins $\mathcal I_1, \ldots, \mathcal I_L$ partitioning
$[0,1]$ with boundary set $\mathcal T$; we set $L = 10$ throughout.
Define
\[
\operatorname{ECE}(c, a)
\;:=\;
\sum_{\ell=1}^L
\big|\mathbb E\!\left[(a - c)\mathbf{1}\{c \in \mathcal I_\ell\}\right]\big|,
\]
where the expectation is over $q \sim \mathcal Q$ and the sampling
randomness within each question. The oracle correctness label is
$a_q^\star := Y_q(z_q^\star)$ and the deployed correctness label is
$\hat a := Y_q(\hat z_N)$. We instantiate $c$ by $\hat c_i$ and by
$c_q^\star$ to obtain the central metrics
\[
\mathrm{Sem}_i\text{-ECE}
\;:=\;
\operatorname{ECE}(\hat c_i, \hat a),
\qquad
\mathrm{ECE}^\star
\;:=\;
\operatorname{ECE}(c_q^\star, a_q^\star),
\]
the calibration error of $\mathrm{Sem}_i$ ($i \in \{1,2\}$) and the
calibration error of the unattainable population-level oracle pair.
The deployment accuracy
$\bar a := \mathbb{E}_q[\hat a] = \mathbb{E}_q[Y_q(\hat z_N)]$ is the
population mean of $\hat a$ and serves as the natural reference for
the leading-order analysis in \Cref{sec:comparison}.
\section{A Semantic-Sampling Framework for Evaluating Calibration}
\label{sec:framework}
\subsection{Same-sample estimator \(\mathrm{Sem}_1\)}
\label{sec:sem1}
The most direct estimate of the oracle confidence
\(c_q^\star=\max_k\pi_{q,k}\) is the empirical maximum on the same block
that produced \(\hat z_N\):
\begin{equation}\label{eq:c1-def}
  \hat c_1
  \;:=\;
  \max_{k\in\mathcal Z_q}\hat\pi_N(k).
\end{equation}
This is the natural plug-in once one has committed to estimating
\(\max_k\pi_{q,k}\) by \(\max_k\hat\pi_N(k)\). We take it as the
same-sample member of our framework and refer to its calibration
error as \(\mathrm{Sem}_1\)-ECE.

\(\hat c_1\) couples two operations on the same block \(N\): it
\emph{selects} the empirical winner and \emph{reports} its empirical
frequency. Because \(\max\) is convex and the empirical PMF
\(\hat\pi_N\) is unbiased for \(\pi_q\), Jensen's inequality gives
\begin{equation}\label{eq:winners-curse}
  \mathbb{E}\!\left[\hat c_1\,\middle|\,q\right]
  \;=\;
  \mathbb{E}\!\left[\max_{k}\hat\pi_N(k)\right]
  \;\ge\;
  \max_{k}\mathbb{E}\!\left[\hat\pi_N(k)\right]
  \;=\;
  c_q^\star,
\end{equation}
with strict inequality whenever the top-two gap
\(\Delta_q := \pi_{q,z_q^\star} - \max_{k\ne z_q^\star}\pi_{q,k}\) is
finite and \(n<\infty\). At the population level \(\hat c_1\) is
biased upward, and the empirical winner is over-represented on the
block that selected it — the classical \emph{winner's curse} \citep{kagel2002common}. The
slack in \eqref{eq:winners-curse} is a finite-sample property of the
same-sample design rather than of the oracle target, and it suggests
a natural alternative: evaluate the chosen answer on samples that did
not participate in selecting it.

\subsection{Held-out estimator \(\mathrm{Sem}_2\)}
\label{sec:sem2}

To tackle this issue, we then introduce \(\mathrm{Sem}_2\), which decouples selection from
evaluation by reading the confidence off the disjoint block \(E\):
\begin{equation}\label{eq:c2-def}
  \hat c_2
  \;:=\;
  \hat\pi_E(\hat z_N).
\end{equation}
The deployed answer is the same \(\hat z_N\); only the way we score
it changes. Because \(E\) is independent of \(N\) and \(\hat\pi_E\)
is unbiased for \(\pi_q\), \(\mathrm{Sem}_2\) satisfies the
\emph{conditional unbiasedness} property
\begin{equation}\label{eq:c2-cond-unbiased}
  \mathbb{E}\!\left[\hat c_2\,\middle|\,q,\,\hat z_N\right]
  \;=\;
  \pi_{q,\hat z_N}:
\end{equation}
given the selection \(\hat z_N\), \(\hat c_2\) targets exactly the
population probability of the selected answer, and the Jensen slack
in \eqref{eq:winners-curse} is eliminated at the conditional level, which is
a property \(\mathrm{Sem}_1\) does not enjoy at any level.

Conditional unbiasedness is for the population probability of the
\emph{empirical} mode \(\hat z_N\), not the \emph{true} mode
\(z_q^\star\). Marginalizing \eqref{eq:c2-cond-unbiased} over
\(\hat z_N\),
\begin{equation}\label{eq:c2-marg-bias}
  \mathbb{E}\!\left[\hat c_2\,\middle|\,q\right]
  \;=\;
  \mathbb{E}\!\left[\pi_{q,\hat z_N}\,\middle|\,q\right]
  \;\le\;
  \pi_{q,z_q^\star}
  \;=\;
  c_q^\star,
\end{equation}
with strict inequality whenever
\(\Pr(\hat z_N\ne z_q^\star\mid q)>0\). \(\mathrm{Sem}_2\) thus trades
\(\mathrm{Sem}_1\)'s upward \emph{Jensen bias} for a downward
\emph{selection bias}. The two biases are mathematically distinct:
\eqref{eq:winners-curse} is a property of the \(\max\) operator on
the selection block, while \eqref{eq:c2-marg-bias} is a property of
the noise in the selection itself. Both are \(O(n^{-1/2})\) on
low-margin questions and vanish in the large-margin limit. Whether
the trade is favorable, and on which version of the calibration
metric, is the subject of \Cref{sec:theory}.

\textbf{The Sem-ECE family.} Pairing each estimator with the population calibration metric gives the two members of the Sem-ECE
family,
\[
  \mathrm{Sem}_1\text{-ECE}
  :=
  \mathrm{ECE}(\hat c_1),
  \qquad
  \mathrm{Sem}_2\text{-ECE}
  :=
  \mathrm{ECE}(\hat c_2),
\]
distinguished by which of \eqref{eq:c1-def}, \eqref{eq:c2-def} is
substituted into the calibration metric. \Cref{alg:semece} states the
framework in \Cref{app:algorithm}. Because both members commit to the same
\(\hat z_N\), the deployment policy is unaffected by the choice of
estimator; the choice influences only the reported confidence and its
calibration error.


\section{Theoretical Analysis}
\label{sec:theory}

We analyze the relationship between the plug-in calibration errors
$\text{Sem}_1$-ECE, $\text{Sem}_2$-ECE, and the oracle
$\text{ECE}^\star$ in two layers. \Cref{sec:unbiased} establishes
asymptotic unbiasedness through a pointwise bias bound
(\Cref{thm:bias}), a binned-ECE bound (\Cref{thm:ece-bound}), and
supporting bounds on the underlying confidence and selection errors
(\Cref{boundc1,thm:bound-c2}); together these yield
$\text{Sem}_i$-ECE $\to \text{ECE}^\star$ as $n, m \to \infty$ and
identify a low-margin regime $\mathcal{Q}_{\mathrm{low}}$ on which the
two estimators do not coincide asymptotically. \Cref{sec:comparison}
resolves the leading $1/\sqrt n$ constant on
$\mathcal{Q}_{\mathrm{low}}$ via a local CLT bias expansion and
identifies a strictly nested Jensen-dominated regime
$\mathcal{Q}_{\mathrm{JDR}}$ on which $\text{Sem}_2$-ECE is closer
to $\text{ECE}^\star$ than $\text{Sem}_1$-ECE. All proofs are in
\Cref{app:proofs}.

\subsection{Asymptotic Unbiasedness}\label{sec:unbiased}

\Cref{thm:bias} controls the per-question bias
of $\hat c_i$ about $c_q^\star$ and identifies the regime separation.
\Cref{thm:ece-bound,boundc1,thm:bound-c2} then lift this pointwise
control to the binned ECE through two estimator-level errors: a
confidence error $\varepsilon_n$ and a selection error $\delta_n$.

\begin{theorem}[Pointwise bias bound]\label{thm:bias}
For each $q$ with $\Delta_q > 0$ and $i \in \{1,2\}$,
\begin{equation}\label{eq:confidence-bound}
\big|\mathbb{E}[\hat c_i \mid q] - c_q^\star\big|
\;\le\;
\min\!\left\{C
\sqrt{\tfrac{\log 2K_q}{2n}},
\;\;
(K_q-1)\exp\!\Big(-\tfrac{n\Delta_q^2}{2 p_q}\Big)
\right\},
\end{equation}
where $C$ is a universal constant.
\end{theorem}

The right-hand side has two complementary terms. The first term (Hoeffding) gives a uniform $n^{\frac{1}{2}}$ ceiling regardless of margin. The second term (Bernstein) term shrinks exponentially in the standardized margin $\tilde m_q^2 = n\Delta_q^2/p_q$ and sharpens the bound when the margin is large. The two cross at $\tilde m_q^2 \asymp \log K_q$, splitting questions into a \emph{large-margin} regime
($\tilde m_q^2 \ge \log K_q$, Bernstein wins, bias $o(n^{-1/2})$) and
a \emph{low-margin} regime ($\tilde m_q^2 < \log K_q$, Hoeffding is
tight, bias $\Theta(n^{-1/2})$).


\textbf{Confidence and selection errors.}
For an estimator $\hat c \in \{\hat c_1, \hat c_2\}$ paired with the
deployed correctness label $\hat a$, define
\[
\varepsilon_n := \mathbb{E}\,|\hat c - c_q^\star|
\quad (\text{confidence error}),
\quad
\delta_n := \Pr\!\big(\hat a \ne a_q^\star\big)
\quad (\text{selection error}).
\]
With $\operatorname{ECE}(\cdot, \cdot)$ the binned ECE operator as defined in
\Cref{sec:calibration-metric}, \Cref{thm:ece-bound} converts
$(\varepsilon_n, \delta_n)$ into a bound on
$|\operatorname{ECE}(\hat c, \hat a) - \text{ECE}^\star|$;
\Cref{boundc1,thm:bound-c2} bound $(\varepsilon_n, \delta_n)$ for
each estimator.

\begin{theorem}[Bin-ECE Bounds]\label{thm:ece-bound}
For any $\eta>0$,
\begin{equation}\label{eq:ece-bound}
\left|
\operatorname{ECE}(\hat c,\hat a)
-
\operatorname{ECE}(c_q^\star,a_q^\star)
\right|
\le
\delta_n+\varepsilon_n
+
2\left\{
\frac{\varepsilon_n}{\eta}
+
\mathbb P\bigl(\operatorname{dist}(c_q^\star,\mathcal T)\le \eta\bigr)
\right\}.
\end{equation}
If, in addition, $c_q^\star$ has density bounded by $M$, then
\[
\left|
\operatorname{ECE}(\hat c,\hat a)
-
\operatorname{ECE}(c_q^\star,a_q^\star)
\right|
\le
\delta_n+\varepsilon_n
+
4\sqrt{2M(L-1)\varepsilon_n}.
\]
\end{theorem}

Therefore, the gap between an estimated ECE and the true ECE is
bounded in terms of the confidence error and selection error.

\begin{theorem}[Bounds for confidence and selection errors for $\hat c_1$]\label{boundc1}
Let $\hat c=\hat c_1$ and $\hat a=Y_q(\hat z_N)$,
\[
\varepsilon_n
\le
\mathbb E_q
\left[
\min\left\{
\sqrt{\frac{\log(2K_q)}{2n}},
\sqrt{\frac{\pi_{q,1}(1-\pi_{q,1})}{n}}
+
(K_q-1)\exp\left(
-\frac{n\Delta_q^2}{2p_q}
\right)
\right\}
\right],
\]
and
\[
\delta_n
\le
\mathbb E_q
\left[
(K_q-1)\exp\left(
-\frac{n\Delta_q^2}{2p_q}
\right)
\right].
\]
\end{theorem}

Using Sem$_1$-ECE, both the confidence error and selection error
converge to $0$ as $n\rightarrow \infty$. Next, we provide the bound
for Sem$_2$-ECE.

\begin{theorem}[Bounds for confidence and selection errors for $\hat c_2$]\label{thm:bound-c2}
Let $\hat c=\hat c_2$ and $\hat a=Y_q(\hat z_N)$,
\[
\varepsilon_{n}
\le
\mathbb E_q
\left[
\sqrt{\frac{\pi_{q,1}(1-\pi_{q,1})}{m}}
+
(K_q-1)
\exp\left(
-\frac{n\Delta_q^2}{2p_q}
\right)
\right],
\]
and $\delta_n$ has the same upper bound as in \Cref{boundc1}.
\end{theorem}

Therefore, both Sem$_1$-ECE and Sem$_2$-ECE asymptotically converge
to the true ECE. The bounds differ only in the sample for the Bernoulli term:
$\mathrm{Sem}_1$ reuses block $N$ at rate $1/\sqrt n$,
$\mathrm{Sem}_2$ averages over $E$ at $1/\sqrt m$;
\Cref{boundc1} additionally retains the Hoeffding term as a
margin-free fallback.

Combining \Cref{thm:ece-bound,boundc1,thm:bound-c2}, convergence is
exponentially fast on the large-margin regime
$\{q : \tilde m_q^2 \geq \log K_q\}$. The residual separation
concentrates on the low-margin regime
$\mathcal{Q}_{\mathrm{low}} := \{q : \tilde m_q^2 < \log K_q\}$,
resolved next to leading order.

\subsection{ECE Comparison in the Low-Margin Regime}\label{sec:comparison}

On $\mathcal{Q}_{\mathrm{low}}$, both $\text{Sem}_1$ and
$\text{Sem}_2$ have $\Theta(n^{-1/2})$ bias by \Cref{thm:bias}. To
compare the two estimators we decompose the bias by sign. Both
biases originate from the event $\{\hat z_N \ne z_q^\star\}$, but on
this event $\text{Sem}_1$ overshoots $c_q^\star$ (Jensen's
inequality, $\mathbb{E}[\max_z \hat\pi_N(z) \mid q] \ge c_q^\star$),
while $\text{Sem}_2$ undershoots ($\hat\pi_E(\hat z_N)$ is the
held-out frequency of a non-modal class).

\textbf{Bin-interior assumption.}
Throughout \Cref{sec:comparison} we assume $c_q^\star$ is bounded
away from the bin boundaries $\mathcal T$ on a vanishing-measure set:
\begin{equation}\label{eq:bin-interior}
\Pr_q\!\big(\mathrm{dist}(c_q^\star, \mathcal{T}) \le \eta_n\big)
= o(n^{-1/2}),
\quad
\sqrt n\,\eta_n \to \infty.
\end{equation}
Under \eqref{eq:bin-interior}, $\hat c_1, \hat c_2$ lie in the same
bin as $c_q^\star$ with probability $1 - o(n^{-1/2})$, the
binning-instability term in \Cref{thm:ece-bound} contributes at order
$o(n^{-1/2})$ rather than at the unconditional $O(n^{-1/4})$
ceiling, and the leading-order behavior of $\text{Sem}_i$-ECE on
$\mathcal{Q}_{\mathrm{low}}$ is governed by the conditional bias
expansion below.

\textbf{Bias expansion.}
For $q \in \mathcal{Q}_{\mathrm{low}}$, a local CLT at the boundary
$\hat\pi_N(z_q^\star) = \hat\pi_N(z_q^{(2)})$ with a folded-normal
calculation (\Cref{app:bias-expansion}) gives, to leading order,
\begin{equation}\label{eq:bias-expansion}
\mathbb{E}[\hat c_1 - c_q^\star \mid q]
= \tfrac{\sqrt{p_q}}{\sqrt n}\,J(\tilde\lambda_q) + o(n^{-1/2}),
\quad
\mathbb{E}[\hat c_2 - c_q^\star \mid q]
= -\tfrac{\sqrt{p_q}}{\sqrt n}\,S(\tilde\lambda_q) + o(n^{-1/2}),
\end{equation}
with positive Jensen bias
$J(\tilde\lambda) := \varphi(2\tilde\lambda) - 2\tilde\lambda\Phi(-2\tilde\lambda)$
(the leading-order winner's curse
from~\eqref{eq:winners-curse}) and selection bias
$S(\tilde\lambda) := 2\tilde\lambda\Phi(-2\tilde\lambda)$.
The ECE-level comparisons are governed by
\begin{equation}\label{eq:gA-gB}
g_A(\tilde\lambda) := J + S = \varphi(2\tilde\lambda),
\qquad
g_B(\tilde\lambda) := J - S
= \varphi(2\tilde\lambda) - 4\tilde\lambda\Phi(-2\tilde\lambda),
\end{equation}
where $g_A > 0$ on $\mathcal{Q}_{\mathrm{low}}$, and $g_B$ is
strictly decreasing on $[0,\infty)$ with unique positive root
$\tilde\lambda^\star \approx 0.306$, i.e., $\tilde m^\star
= 2\tilde\lambda^\star \approx 0.612$
(\Cref{fig:regime}b; uniqueness proof in \Cref{app:lambda-star}).
\begin{definition}[Jensen-dominated regime]\label{def:jdr}
The \emph{Jensen-dominated regime} (JDR) is
$\mathcal{Q}_{\mathrm{JDR}}
:= \{q : \tilde\lambda_q < \tilde\lambda^\star\}
\subsetneq \mathcal{Q}_{\mathrm{low}}$.
\end{definition}

\begin{figure}[t]
\vspace{-0.5cm}
\centering
\small
\includegraphics[width=0.48\linewidth]{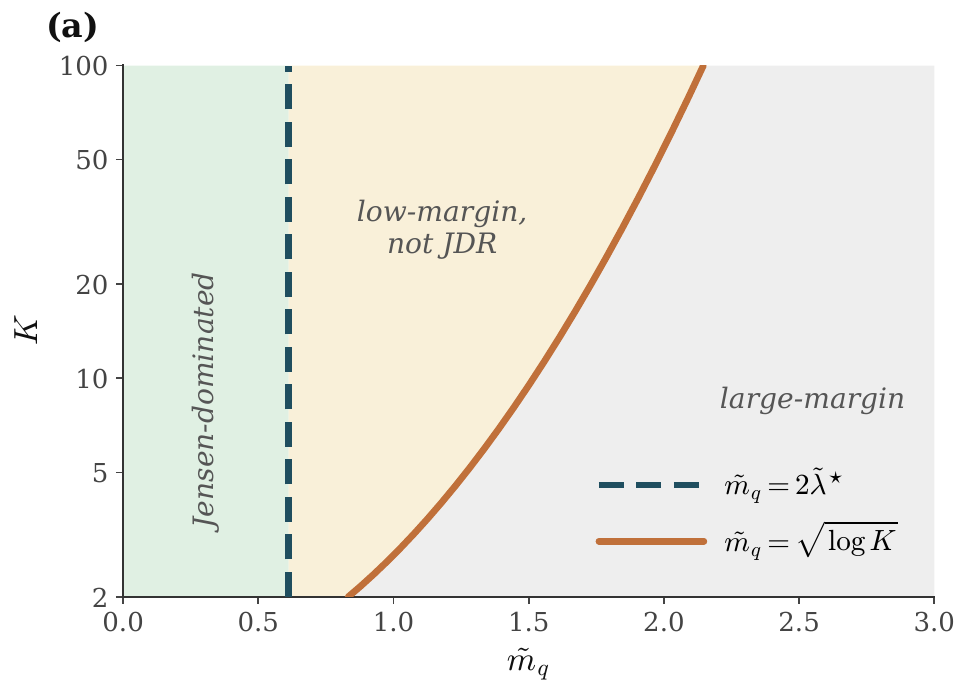}
\hfill
\includegraphics[width=0.48\linewidth]{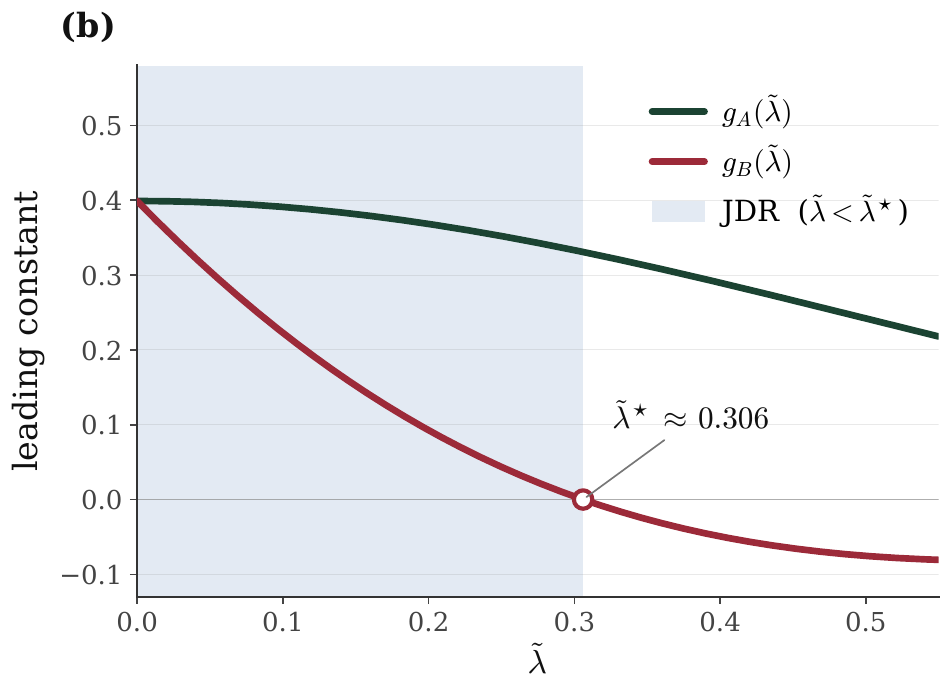}
\vspace{-0.3cm}
\caption{(a) Regime diagram on the $(\tilde m_q, K_q)$ plane,
partitioned by the JDR boundary $\tilde m_q = 2\tilde\lambda^\star$
(\Cref{thm:sharp}, dashed) and the crossover
$\tilde m_q = \sqrt{\log K_q}$ (\Cref{thm:bias}, solid). In JDR
(green), $\mathrm{Sem}_2$ wins on both raw ECE and oracle distance;
in the intermediate band (yellow), $\mathrm{Sem}_2$ has smaller raw
ECE but is farther from the oracle; in the large-margin region
(gray), the two estimators are asymptotically indistinguishable.
(b) Leading constants from~\eqref{eq:gA-gB}:
$g_A(\tilde\lambda) = \varphi(2\tilde\lambda) > 0$ everywhere, while
$g_B(\tilde\lambda) = \varphi(2\tilde\lambda) - 4\tilde\lambda\Phi(-2\tilde\lambda)$
is positive only on $\tilde\lambda < \tilde\lambda^\star \approx 0.306$
(shaded).}
\label{fig:regime}
\end{figure}

\textbf{Direct ECE gap.}
Adding the two expansions in \eqref{eq:bias-expansion} sign-aligns
the biases: $\text{Sem}_1$ overshoots $c_q^\star$ by
$J(\tilde\lambda_q)$ and $\text{Sem}_2$ undershoots by
$S(\tilde\lambda_q)$. Under over-confidence, the absolute deviations
of the two estimators from $\bar a$ shift by $-J$ and $+S$
respectively, producing a gap of $J + S = g_A$.

\begin{theorem}[Direct ECE gap]\label{thm:gap}
Suppose $\mathbb{E}_q[c_q^\star] - \bar a > 0$ on
$\mathcal{Q}_{\mathrm{low}}$, where $\bar a := \mathbb{E}_q[\hat a]$
is the deployment accuracy. Then
\begin{equation}\label{eq:gap}
\text{Sem}_1\text{-ECE} - \text{Sem}_2\text{-ECE}
= \tfrac{1}{\sqrt n}\,
\mathbb{E}_q\!\left[\sqrt{p_q}\,g_A(\tilde\lambda_q)\right]
+ o(n^{-1/2}) \;>\; 0.
\end{equation}
\end{theorem}

\Cref{thm:gap} requires only over-confidence on
$\mathcal{Q}_{\mathrm{low}}$ --- no constraint on $\tilde\lambda_q$
within the regime: $\text{Sem}_2$-ECE is strictly smaller than
$\text{Sem}_1$-ECE throughout $\mathcal{Q}_{\mathrm{low}}$ at order
$n^{-1/2}$.

\textbf{Oracle ECE distance.}
\Cref{thm:gap} compares the two raw ECE values but does not say
which is closer to $\text{ECE}^\star$. A smaller raw ECE for
$\text{Sem}_2$ is consistent with two structural scenarios:
(a) $\text{Sem}_2$-ECE sits between $\text{Sem}_1$-ECE and
$\text{ECE}^\star$ (closer to oracle); or (b)
$\text{Sem}_2$-ECE has overshot through $\text{ECE}^\star$ to
the opposite side (farther from oracle, but smaller in absolute
distance to $\bar a$). Distinguishing (a) from (b) hinges on $|J|$
versus $|S|$, i.e., on $g_B$.

\begin{theorem}[Sharp oracle ECE distance on JDR]\label{thm:sharp}
Suppose $\mathbb{E}_q[c_q^\star] - \bar a \ge c_0 / \sqrt n$ for some
fixed $c_0 > 0$ (non-degenerate over-confidence on
$\mathcal{Q}_{\mathrm{low}}$), and the population is supported in
$\mathcal{Q}_{\mathrm{JDR}}$, i.e.,
$\sup_q \tilde\lambda_q < \tilde\lambda^\star$. Then
\begin{equation}\label{eq:sharp}
\big|\text{Sem}_1\text{-ECE} - \text{ECE}^\star\big|
- \big|\text{Sem}_2\text{-ECE} - \text{ECE}^\star\big|
= \tfrac{1}{\sqrt n}\,
\mathbb{E}_q\!\left[\sqrt{p_q}\,g_B(\tilde\lambda_q)\right]
+ o(n^{-1/2}) \;>\; 0.
\end{equation}
\end{theorem}

\Cref{thm:sharp} adds two assumptions to \Cref{thm:gap}: (i) JDR,
ensuring $g_B(\tilde\lambda_q) > 0$ pointwise; and (ii)
non-degenerate over-confidence, ensuring the absolute values in
\eqref{eq:sharp} do not flip sign within the $o(n^{-1/2})$ remainder.
The conclusion is correspondingly stronger: $\text{Sem}_2$-ECE is
closer to $\text{ECE}^\star$ than $\text{Sem}_1$-ECE is, not
merely smaller in absolute terms.

\textbf{Regime structure.}
Combining \Cref{thm:bias,thm:ece-bound,boundc1,thm:bound-c2,thm:gap,thm:sharp}
(\Cref{fig:regime}a), the population $\mathcal{Q}$ partitions into
three regimes:
\begin{itemize}
\item \emph{Large-margin} ($\tilde m_q^2 \ge \log K_q$):
\Cref{boundc1,thm:bound-c2} give $\varepsilon_n, \delta_n$
exponentially small; via \Cref{thm:ece-bound},
$|\text{Sem}_i$-ECE $- \text{ECE}^\star|$ is exponentially
small, and $\text{Sem}_1, \text{Sem}_2$ are asymptotically
indistinguishable.
\item \emph{Low-margin, not JDR}
($2\tilde\lambda^\star \le \tilde m_q < \sqrt{\log K_q}$):
\Cref{thm:gap} applies, so $\text{Sem}_2$-ECE $<$
$\text{Sem}_1$-ECE; but $g_B(\tilde\lambda_q) < 0$, so
$\text{Sem}_2$-ECE is \emph{farther} from $\text{ECE}^\star$
than $\text{Sem}_1$-ECE. The smaller raw ECE is achieved by
overshooting through $\text{ECE}^\star$.
\item \emph{JDR} ($\tilde m_q < 2\tilde\lambda^\star$): both
\Cref{thm:gap,thm:sharp} apply; $\text{Sem}_2$ wins on both
metrics.
\end{itemize}
\section{Experiments}
\label{sec:experiments}

We evaluate the Sem-ECE framework on three open-ended QA benchmarks
across five frontier LLMs, with three goals: (i) verify the
asymptotic predictions of \Cref{sec:theory} on real data
(\Cref{sec:exp_asymptotic,sec:exp_sharp}); (ii) compare against the
verbalized-confidence baseline at the per-pair level
(\Cref{sec:exp_main_results}); and (iii) characterize what the
resulting reliability diagrams reveal about semantic agreement versus
factual accuracy (\Cref{app:reliability}). Per-model breakdowns, the 
boundary-alignment numerics and the bootstrap protocols are in
\Cref{app:additional-figures,app:alignment,app:bootstrap}.

\subsection{Setup}
\label{sec:exp_setup}

\textbf{Datasets and models.}
\textbf{SimpleQA}~\citep{wei2024simpleqa} (short-form factoid),
\textbf{HLE}~\citep{phan2025hle} (expert-level multidisciplinary),
and \textbf{PopQA}~\citep{mallen2023popqa} (long-tail entity-centric).
We evaluate five commercial models accessed via their respective
APIs: OpenAI \texttt{gpt-5.4}~\citep{openai2025gpt5},
Anthropic \texttt{claude-opus-4.6}~\citep{anthropic2025claude},
Google \texttt{gemini-3.1-flash-lite-preview}~\citep{google2025gemini},
xAI \texttt{grok-4.20-0309}~\citep{xai2025grok} (non-reasoning),
and Mistral \texttt{mistral-large-latest}~\citep{mistral2024large}.
The 15 model--benchmark pairs form the evaluation grid; we draw
$n_{\max} = 50$ stochastic generations per question.

\textbf{Pipeline.}
For each question we (i) generate $n_{\max}$ responses, (ii) cluster
them into semantic answer classes via an LLM judge
(\texttt{gpt-5.4}), and (iii) grade each response against the
reference answer.

\textbf{Confidence sources.}
\textbf{Sem}$_1$: $\hat c_1 = \max_{z}\hat\pi_{[n_{\max}]}(z)$,
computed on the full pool of $n_{\max}=50$ generations.
\textbf{Sem}$_2$: $\hat c_2 = \hat\pi_{E}(\hat z_N)$ at $n=m=25$,
averaged over $R=10$ random half-splits $(N, E)$ of the same pool
(so the two estimators share underlying samples but use them
differently). \textbf{Verbalized confidence (Ver)}: elicited via
"\texttt{Confidence: X\%}", parsed from each of the $n_{\max}$
generations and averaged; parse failures imputed at $1.0$
~\citep{tian2023just,xiong2024can}.

\textbf{Metrics.}
$\text{Sem}_i$-ECE and Ver-ECE with $L = 10$ equal-width bins
(\Cref{sec:calibration-metric}). Stratification by margin uses
$\Delta_q$ as the regime axis; under the $p_q \to 1$ convention of
\Cref{sec:empirical-access}, $\tilde m_q = \sqrt n\,\Delta_q$, so the
regime boundaries from \Cref{thm:bias,thm:sharp} are visible directly
on the $\Delta_q$-axis.

\subsection{Asymptotic convergence}
\label{sec:exp_asymptotic}

\Cref{thm:bias,thm:ece-bound,boundc1,thm:bound-c2} predict that
$\text{Sem}_1$-ECE and $\text{Sem}_2$-ECE converge to a common
limit $\mathrm{ECE}^\star$ as $n \to \infty$. \Cref{fig:n-sweep} sub-samples each question's $n_{\max} = 50$ semantic-class assignments
down to $n \in \{10, 20, 30, 40, 50\}$ and recomputes pooled
$\text{Sem}_i$-ECE on each benchmark. The two curves approach a
common limit on every benchmark from \emph{opposite sides}:
$\text{Sem}_1$ from above (positive Jensen bias), $\text{Sem}_2$
from below (negative selection bias), which is the empirical signature of
the bias decomposition \eqref{eq:bias-expansion}.

\subsection{Sharp comparison: regime structure and rate}
\label{sec:exp_sharp}

\Cref{thm:gap} predicts that on $\mathcal{Q}_{\mathrm{low}}$ the
direct ECE gap is, to leading order in $1/\sqrt n$,
$\mathbb{E}_q[\sqrt{p_q}\,g_A(\tilde\lambda_q)]/\sqrt n$. Under the
$p_q \to 1$ convention, the prediction at any fixed standardized
margin $\tilde m$ collapses to a single number
$\varphi(\tilde m)/\sqrt n$. We test this prediction along three
axes: regime structure, leading constant, and convergence rate.

\textbf{Regime structure.}
\Cref{fig:ece-margin} shows the stratification of pooled $\text{Sem}_1$-ECE and $\text{Sem}_2$-ECE by per-question margin $\Delta_q$. The panels are divided into three regions based on the JDR boundary ($\Delta_q \approx 0.087$) and the low/large boundary ($\Delta_q = \sqrt{\log K_q / n}$) from \Cref{fig:regime}(a). The empirical results match theoretical predictions: separation is greatest below the JDR threshold, diminishes in the intermediate range, and converges above the low/large boundary.

\textbf{Leading constant and convergence rate.}
With \emph{no fitted constants}, the leading-order prediction
$\varphi(\tilde m^\star)/\sqrt n$ recovers the empirical
Sem$_1$-ECE $-$ Sem$_2$-ECE gap to within $11$--$27\%$ at both
regime boundaries on every benchmark, with empirical/theory
ratios consistently above $1$ at the JDR boundary and below $1$
at the low/large boundary (\Cref{app:alignment}).
On the low-margin sub-population, the gap shrinks at the
predicted $n^{-1/2}$ rate, with fitted log-log slopes within
$0.08$ of $-0.50$ across all three benchmarks
(\Cref{fig:rate}, \Cref{app:additional-figures}). The
sign-consistent boundary residual and the steeper-than-$-0.50$
slope direction are both consistent with a subleading $O(1/n)$
Edgeworth correction.

\begin{figure}[t]
\centering
\small
\includegraphics[width=\linewidth]{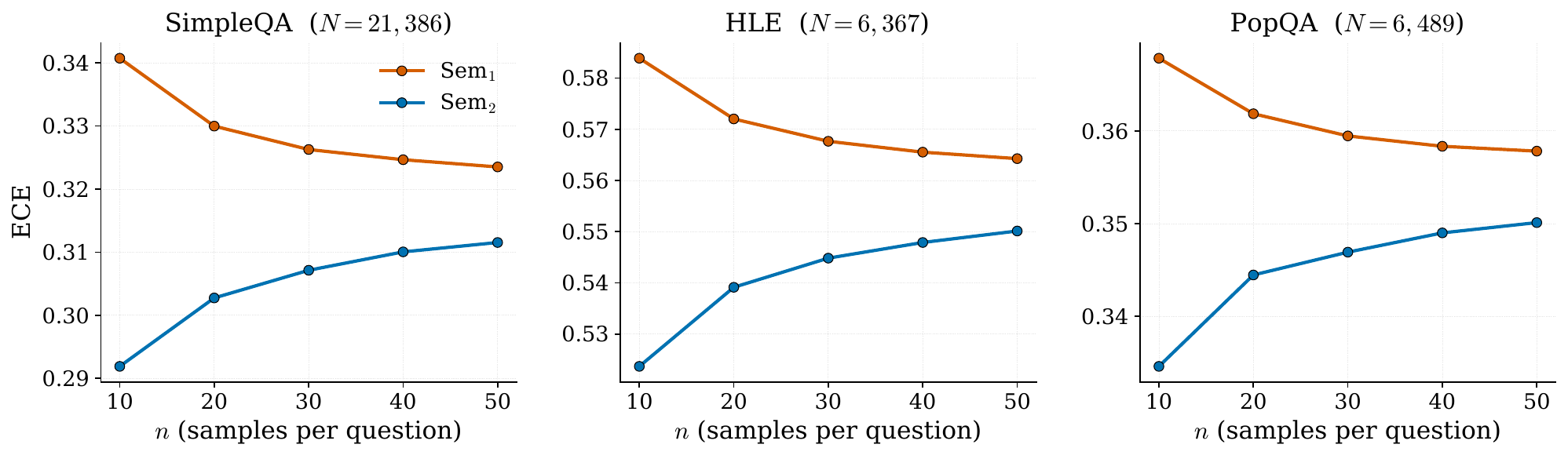}\vspace{-0.5cm}
\caption{Pooled $\text{Sem}_1$-ECE (orange) and $\text{Sem}_2$-ECE
(blue) as functions of the per-question budget $n \in [10, 50]$.
The two curves converge to a common limit on every benchmark from
opposite sides — the empirical signature of \Cref{thm:bias,thm:ece-bound}
via~\eqref{eq:bias-expansion}.}
\label{fig:n-sweep}
\end{figure}

\begin{figure}[t]
\centering
\small
\includegraphics[width=\linewidth]{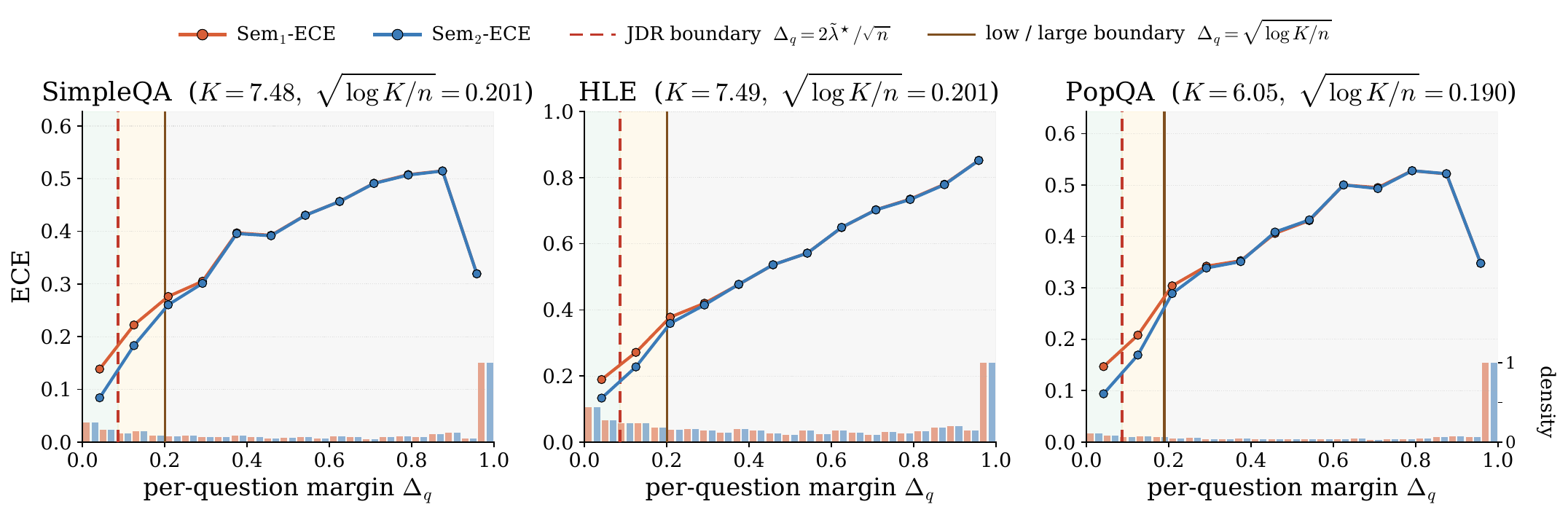}
\vspace{-0.5cm}
\caption{Pooled $\text{Sem}_1$-ECE and $\text{Sem}_2$-ECE
stratified by per-question margin $\Delta_q$ on SimpleQA (left),
HLE (middle), PopQA (right). The dashed red line marks the JDR
boundary $\Delta_q = 2\tilde\lambda^\star/\sqrt n$ and the solid
brown the low/large boundary $\Delta_q = \sqrt{\log K_q/n}$,
partitioning each panel into the three regions of
\Cref{fig:regime}(a). The two metrics overlap above the low/large
boundary and separate below, with the largest gap below the JDR
boundary, matching \Cref{thm:bias,thm:gap,thm:sharp}.}
\label{fig:ece-margin}
\end{figure}

\subsection{Cross-benchmark calibration}
\label{sec:exp_main_results}

\Cref{tab:per_pair_ece} reports per-pair binned ECE for the three
confidence sources.  $\text{Sem}_2$-ECE is no larger than
$\text{Sem}_1$-ECE on \emph{all 15 pairs}, which is an empirical
demonstration of \Cref{thm:gap} extending beyond the strict
low-margin regime. Among the three sources, $\text{Sem}_2$ is
lowest on $12$ pairs, Ver on $3$, and $\text{Sem}_1$ on none. The
strongest aggregate advantage of $\text{Sem}_2$ appears on HLE,
where it is the best-performing confidence source for all five
providers. Pooled reliability diagrams are shown in \Cref{fig:reliability};
see \Cref{app:reliability} for detailed discussion. 
A paired bootstrap ($B=1000$) confirms $\mathbb{E}[\hat c_1 - \hat c_2] > 0$ on all 15 pairs and the population-level ECE gap on 11 of 15
(\Cref{app:bootstrap}).

\begin{table}[t]
\centering
\caption{Per-pair binned ECE with $L = 10$ equal-width bins. Lower
ECE is better. Among $\text{Sem}_1$ and $\text{Sem}_2$ only, the
winning cell is highlighted in pink and the loser in blue;
Ver-ECE appears in grey baseline. Across all three sources, the
smallest ECE in each row is in \textbf{bold} and the
second-smallest is \underline{underlined}. $N$ is the number of
jointly clustered and graded questions; $\mathit{Acc}$ is per-pair
accuracy. $\text{Sem}_2$-ECE $\le \text{Sem}_1$-ECE in all $15$
pairs and is the lowest of the three sources in $12$ of $15$.}
\label{tab:per_pair_ece}
\setlength{\tabcolsep}{6pt}
\small
\begin{tabular}{l|lcc|ccc}
\toprule
Benchmark & Provider & $N$ & $\mathit{Acc}$ & Ver-ECE & Sem$_1$-ECE & Sem$_2$-ECE \\
\midrule
\multirow{5}{*}{SimpleQA}
 & OpenAI    & $4315$ & $0.336$ & \cellcolor{baselinegrey}$0.5151$ & \cellcolor{loserblue}$\underline{0.1982}$ & \cellcolor{winnerpink}$\bm{0.1786}$ \\
 & Anthropic & $4290$ & $0.482$ & \cellcolor{baselinegrey}$\bm{0.2030}$ & \cellcolor{loserblue}$0.3525$ & \cellcolor{winnerpink}$\underline{0.3451}$ \\
 & Gemini    & $4288$ & $0.609$ & \cellcolor{baselinegrey}$0.3769$ & \cellcolor{loserblue}$\underline{0.1808}$ & \cellcolor{winnerpink}$\bm{0.1724}$ \\
 & xAI       & $4169$ & $0.229$ & \cellcolor{baselinegrey}$0.6062$ & \cellcolor{loserblue}$\underline{0.3646}$ & \cellcolor{winnerpink}$\bm{0.3485}$ \\
 & Mistral   & $4324$ & $0.260$ & \cellcolor{baselinegrey}$0.6020$ & \cellcolor{loserblue}$\underline{0.5218}$ & \cellcolor{winnerpink}$\bm{0.5117}$ \\
\midrule
\multirow{5}{*}{HLE}
 & OpenAI    & $2128$ & $0.101$ & \cellcolor{baselinegrey}$0.6804$ & \cellcolor{loserblue}$\underline{0.4696}$ & \cellcolor{winnerpink}$\bm{0.4522}$ \\
 & Anthropic & $1486$ & $0.205$ & \cellcolor{baselinegrey}$\underline{0.4563}$ & \cellcolor{loserblue}$0.4574$ & \cellcolor{winnerpink}$\bm{0.4403}$ \\
 & Gemini    & $2111$ & $0.151$ & \cellcolor{baselinegrey}$0.7917$ & \cellcolor{loserblue}$\underline{0.5362}$ & \cellcolor{winnerpink}$\bm{0.5227}$ \\
 & xAI       & $2005$ & $0.069$ & \cellcolor{baselinegrey}$0.7317$ & \cellcolor{loserblue}$\underline{0.5984}$ & \cellcolor{winnerpink}$\bm{0.5840}$ \\
 & Mistral   & $2128$ & $0.053$ & \cellcolor{baselinegrey}$0.8143$ & \cellcolor{loserblue}$\underline{0.6868}$ & \cellcolor{winnerpink}$\bm{0.6749}$ \\
\midrule
\multirow{5}{*}{PopQA}
 & OpenAI    & $1997$ & $0.399$ & \cellcolor{baselinegrey}$0.3759$ & \cellcolor{loserblue}$\underline{0.2721}$ & \cellcolor{winnerpink}$\bm{0.2612}$ \\
 & Anthropic & $499$  & $0.505$ & \cellcolor{baselinegrey}$\bm{0.2765}$ & \cellcolor{loserblue}$0.3968$ & \cellcolor{winnerpink}$\underline{0.3893}$ \\
 & Gemini    & $1993$ & $0.508$ & \cellcolor{baselinegrey}$0.4585$ & \cellcolor{loserblue}$\underline{0.3130}$ & \cellcolor{winnerpink}$\bm{0.3069}$ \\
 & xAI       & $1896$ & $0.409$ & \cellcolor{baselinegrey}$0.4308$ & \cellcolor{loserblue}$\underline{0.3133}$ & \cellcolor{winnerpink}$\bm{0.3016}$ \\
 & Mistral   & $2000$ & $0.389$ & \cellcolor{baselinegrey}$\bm{0.4722}$ & \cellcolor{loserblue}$0.4795$ & \cellcolor{winnerpink}$\underline{0.4735}$ \\
\bottomrule
\end{tabular}
\end{table}

\textbf{The verbalized exception is Sem-ECE's audit value.}
The 3 cells where Ver beats Sem (Anthropic on SimpleQA/PopQA,
Mistral on PopQA) share the same pattern: high agreement
consistency yet low accuracy, so $\mathbb{E}_q[c_q^\star] \gg \bar a$
(e.g.\ Anthropic SimpleQA: $\bar a = 0.482$ vs.\ Sem$_1$ mean
confidence $0.835$). Verbalized self-moderation happens to land
closer to $\bar a$ in these cells, but a practitioner relying on
Ver alone has no external reference to detect such miscalibration.
Sem-ECE supplies that reference, depending only on sample
frequencies and external accuracy judgments and placing all
providers on the same footing without trusting any model's
self-report; it complements Ver rather than replacing it.

\begin{figure}[t]
\centering
\includegraphics[width=\linewidth]{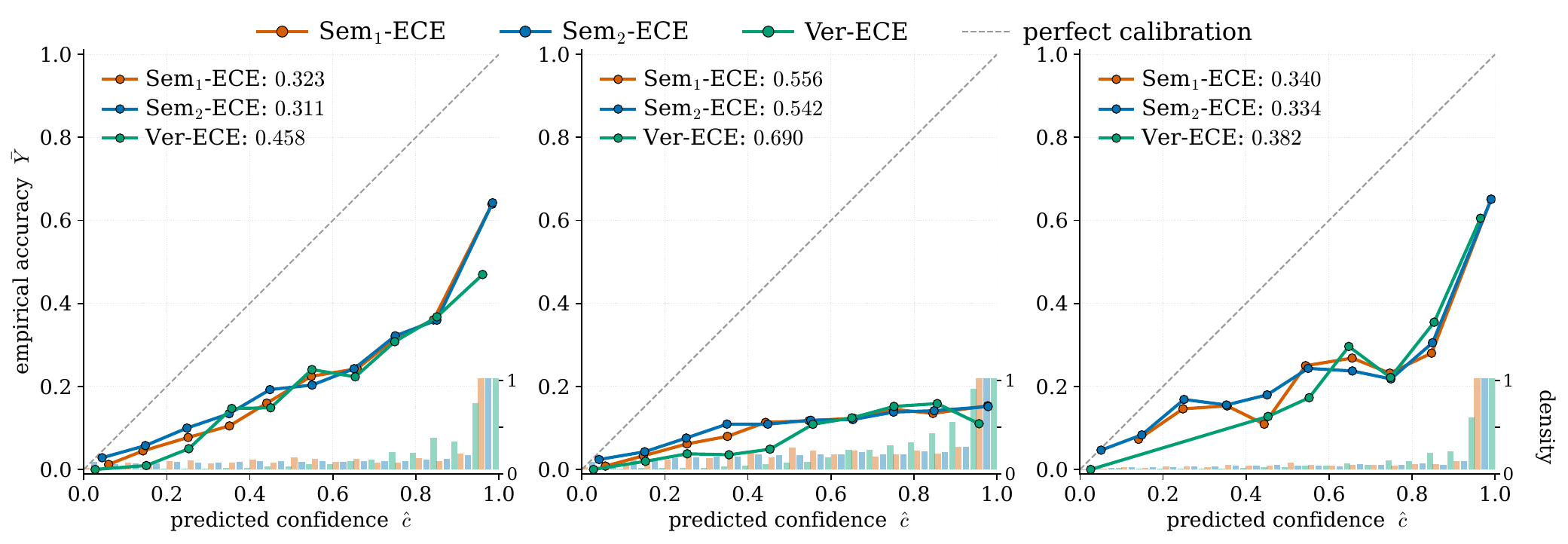}
\vspace{-0.5cm}
\caption{Reliability diagrams pooled across models on SimpleQA
(left), HLE (middle), PopQA (right). Pooled ECE values appear in
each legend; the dashed diagonal is perfect calibration.
$\text{Sem}_2$ achieves the lowest pooled ECE on every benchmark.}
\label{fig:reliability}
\end{figure}
\section{Conclusion}
\label{sec:conclusion}
We introduced \textbf{Sem-ECE}, a calibration evaluation framework for black-box open-ended QA that turns repeated free-form generations into semantic answer classes and uses the resulting frequencies as confidence. Within this framework, we studied two estimators: Sem$_1$-ECE, the standard same-sample self-consistency score, and Sem$_2$-ECE, a held-out variant that separates answer selection from confidence evaluation. We proved both are asymptotically unbiased, and further showed that they agree on easy questions but diverge on hard ones with Sem$_2$ achieving strictly smaller calibration error, so the Sem$_1$--Sem$_2$ gap also serves as a diagnostic for question difficulty. Experiments on three challenging open-ended QA benchmarks across five leading commercial LLMs confirm these predictions and show that Sem-ECE outperforms verbalized confidence and existing sampling-based methods, while complementing logit-based evaluation when internal probabilities are unavailable. 
\section*{Acknowledgment}
This work was supported in part by NIH grant P30 AG073105 and a PSOM AI2D Seeding Project.

\bibliographystyle{unsrtnat}
\bibliography{references}

@book{kagel2002common,
  title     = {Common Value Auctions and the Winner's Curse},
  author    = {Kagel, John H. and Levin, Dan},
  year      = {2002},
  publisher = {Princeton University Press},
  address   = {Princeton, NJ}
}

@inproceedings{guo2017calibration,
  title     = {On Calibration of Modern Neural Networks},
  author    = {Guo, Chuan and Pleiss, Geoff and Sun, Yu and Weinberger, Kilian Q.},
  booktitle = {Proceedings of the 34th International Conference on Machine Learning},
  series    = {Proceedings of Machine Learning Research},
  volume    = {70},
  pages     = {1321--1330},
  year      = {2017},
  publisher = {PMLR}
}

@article{lin2022teaching,
  title   = {Teaching Models to Express Their Uncertainty in Words},
  author  = {Lin, Stephanie C. and Hilton, Jacob and Evans, Owain},
  journal = {Transactions on Machine Learning Research},
  year    = {2022},
  url     = {https://openreview.net/forum?id=8s8K2UZGTZ}
}

@misc{kadavath2022language,
  title         = {Language Models (Mostly) Know What They Know},
  author        = {Kadavath, Saurav and Conerly, Tom and Askell, Amanda and Henighan, Tom and Drain, Dawn and Perez, Ethan and Schiefer, Nicholas and Hatfield-Dodds, Zac and DasSarma, Nova and Tran-Johnson, Eli and Johnston, Scott and El-Showk, Sheer and Jones, Andy and Elhage, Nelson and Hume, Tristan and Chen, Anna and Bai, Yuntao and Bowman, Sam and Fort, Stanislav and Ganguli, Deep and Hernandez, Danny and Jacobson, Josh and Kernion, Jackson and Kravec, Shauna and Lovitt, Liane and Ndousse, Kamal and Olsson, Catherine and Ringer, Sam and Amodei, Dario and Brown, Tom and Clark, Jack and Joseph, Nicholas and Mann, Ben and McCandlish, Sam and Olah, Chris and Kaplan, Jared},
  year          = {2022},
  eprint        = {2207.05221},
  archivePrefix = {arXiv},
  primaryClass  = {cs.CL}
}

@inproceedings{tian2023just,
  title     = {Just Ask for Calibration: Strategies for Eliciting Calibrated Confidence Scores from Language Models Fine-Tuned with Human Feedback},
  author    = {Tian, Katherine and Mitchell, Eric and Zhou, Allan and Sharma, Archit and Rafailov, Rafael and Yao, Huaxiu and Finn, Chelsea and Manning, Christopher D.},
  booktitle = {Proceedings of the 2023 Conference on Empirical Methods in Natural Language Processing},
  pages     = {5433--5442},
  year      = {2023},
  address   = {Singapore},
  publisher = {Association for Computational Linguistics},
  doi       = {10.18653/v1/2023.emnlp-main.330}
}

@inproceedings{lyu2025calibrating,
  title     = {Calibrating Large Language Models with Sample Consistency},
  author    = {Lyu, Qing and Shridhar, Kumar and Malaviya, Chaitanya and Zhang, Li and Elazar, Yanai and Tandon, Niket and Apidianaki, Marianna and Sachan, Mrinmaya and Callison-Burch, Chris},
  booktitle = {Proceedings of the AAAI Conference on Artificial Intelligence},
  volume    = {39},
  number    = {18},
  pages     = {19260--19268},
  year      = {2025},
  doi       = {10.1609/aaai.v39i18.34120}
}

@misc{wei2024simpleqa,
  title         = {Measuring Short-Form Factuality in Large Language Models},
  author        = {Wei, Jason and Nguyen, Karina and Chung, Hyung Won and Jiao, Yunxin Joy and Papay, Spencer and Glaese, Amelia and Schulman, John and Fedus, William},
  year          = {2024},
  eprint        = {2411.04368},
  archivePrefix = {arXiv},
  primaryClass  = {cs.CL}
}

@inproceedings{wang2023selfconsistency,
  title     = {Self-Consistency Improves Chain of Thought Reasoning in Language Models},
  author    = {Wang, Xuezhi and Wei, Jason and Schuurmans, Dale and Le, Quoc V. and Chi, Ed H. and Narang, Sharan and Chowdhery, Aakanksha and Zhou, Denny},
  booktitle = {International Conference on Learning Representations},
  year      = {2023},
  url       = {https://openreview.net/forum?id=1PL1NIMMrw}
}

@article{brier1950verification,
  title   = {Verification of Forecasts Expressed in Terms of Probability},
  author  = {Brier, Glenn W.},
  journal = {Monthly Weather Review},
  volume  = {78},
  number  = {1},
  pages   = {1--3},
  year    = {1950}
}

@inproceedings{naeini2015obtaining,
  title     = {Obtaining Well Calibrated Probabilities Using Bayesian Binning},
  author    = {Naeini, Mahdi Pakdaman and Cooper, Gregory F. and Hauskrecht, Milos},
  booktitle = {Proceedings of the Twenty-Ninth AAAI Conference on Artificial Intelligence},
  pages     = {2901--2907},
  year      = {2015}
}

@article{mielke2022reducing,
  title   = {Reducing Conversational Agents' Overconfidence Through Linguistic Calibration},
  author  = {Mielke, Sabrina J. and Szlam, Arthur and Dinan, Emily and Boureau, Y-Lan},
  journal = {Transactions of the Association for Computational Linguistics},
  volume  = {10},
  pages   = {857--872},
  year    = {2022},
  doi     = {10.1162/tacl_a_00494}
}

@inproceedings{kuhn2023semantic,
  title     = {Semantic Uncertainty: Linguistic Invariances for Uncertainty Estimation in Natural Language Generation},
  author    = {Kuhn, Lorenz and Gal, Yarin and Farquhar, Sebastian},
  booktitle = {International Conference on Learning Representations},
  year      = {2023},
  url       = {https://openreview.net/forum?id=VD-AYtP0dve}
}

@article{farquhar2024detecting,
  title   = {Detecting Hallucinations in Large Language Models Using Semantic Entropy},
  author  = {Farquhar, Sebastian and Kossen, Jannik and Kuhn, Lorenz and Gal, Yarin},
  journal = {Nature},
  volume  = {630},
  pages   = {625--630},
  year    = {2024},
  doi     = {10.1038/s41586-024-07421-0}
}

@incollection{platt1999probabilistic,
  title     = {Probabilistic Outputs for Support Vector Machines and Comparisons to Regularized Likelihood Methods},
  author    = {Platt, John C.},
  booktitle = {Advances in Large Margin Classifiers},
  pages     = {61--74},
  year      = {1999},
  publisher = {MIT Press}
}

@inproceedings{kull2019beyond,
  title     = {Beyond Temperature Scaling: Obtaining Well-Calibrated Multiclass Probabilities with Dirichlet Calibration},
  author    = {Kull, Meelis and Perello-Nieto, Miquel and K{\"a}ngsepp, Markus and Silva Filho, Telmo and Song, Hao and Flach, Peter},
  booktitle = {Advances in Neural Information Processing Systems},
  volume    = {32},
  year      = {2019}
}

@inproceedings{kumar2019verified,
  title     = {Verified Uncertainty Calibration},
  author    = {Kumar, Ananya and Liang, Percy and Ma, Tengyu},
  booktitle = {Advances in Neural Information Processing Systems},
  volume    = {32},
  year      = {2019}
}

@inproceedings{xiao2025restoring,
  title     = {Restoring Calibration for Aligned Large Language Models: A Calibration-Aware Fine-Tuning Approach},
  author    = {Xiao, Jiancong and Hou, Bojian and Wang, Zhanliang and Jin, Ruochen and Long, Qi and Su, Weijie J. and Shen, Li},
  booktitle = {Proceedings of the 42nd International Conference on Machine Learning},
  series    = {Proceedings of Machine Learning Research},
  volume    = {267},
  pages     = {68364--68390},
  year      = {2025},
  publisher = {PMLR},
  url       = {https://proceedings.mlr.press/v267/xiao25b.html}
}

@misc{phan2025hle,
  title         = {Humanity's Last Exam},
  author        = {Phan, Long and Gatti, Alice and Han, Ziwen and Li, Nathaniel
                   and Hu, Josephina and Zhang, Hugh and Shi, Sean and others},
  year          = {2025},
  eprint        = {2501.14249},
  archivePrefix = {arXiv},
  primaryClass  = {cs.LG},
  url           = {https://arxiv.org/abs/2501.14249}
}

@inproceedings{mallen2023popqa,
  title     = {When Not to Trust Language Models: Investigating Effectiveness of
               Parametric and Non-Parametric Memories},
  author    = {Mallen, Alex and Asai, Akari and Zhong, Victor and Das, Rajarshi
               and Khashabi, Daniel and Hajishirzi, Hannaneh},
  booktitle = {Proceedings of the 61st Annual Meeting of the Association for
               Computational Linguistics (Volume 1: Long Papers)},
  pages     = {9802--9822},
  year      = {2023},
  address   = {Toronto, Canada},
  publisher = {Association for Computational Linguistics},
  doi       = {10.18653/v1/2023.acl-long.546}
}

@inproceedings{xiong2024can,
  title     = {Can {LLM}s Express Their Uncertainty? {A}n Empirical Evaluation of
               Confidence Elicitation in {LLM}s},
  author    = {Xiong, Miao and Hu, Zhiyuan and Lu, Xinyang and Li, Yifei
               and Fu, Jie and He, Junxian and Hooi, Bryan},
  booktitle = {The Twelfth International Conference on Learning Representations},
  year      = {2024},
  url       = {https://openreview.net/forum?id=gjeQKFxFpZ}
}

@misc{openai2025gpt5,
  author       = {{OpenAI}},
  title        = {{GPT-5} Models},
  year         = {2025},
  howpublished = {\url{https://platform.openai.com/docs/models}},
  note         = {Accessed: 2026-01}
}

@misc{anthropic2025claude,
  author       = {{Anthropic}},
  title        = {{Claude} Models},
  year         = {2025},
  howpublished = {\url{https://docs.claude.com/en/docs/about-claude/models}},
  note         = {Accessed: 2026-01}
}

@misc{google2025gemini,
  author       = {{Google DeepMind}},
  title        = {{Gemini} API Models},
  year         = {2025},
  howpublished = {\url{https://ai.google.dev/gemini-api/docs/models}},
  note         = {Accessed: 2026-01}
}

@misc{xai2025grok,
  author       = {{xAI}},
  title        = {{Grok} Models},
  year         = {2025},
  howpublished = {\url{https://docs.x.ai/docs/models}},
  note         = {Accessed: 2026-01}
}

@misc{mistral2024large,
  author       = {{Mistral AI}},
  title        = {{Mistral Large}},
  year         = {2024},
  howpublished = {\url{https://docs.mistral.ai/getting-started/models/models_overview/}},
  note         = {Accessed: 2026-01}
}

\appendix

\section{Proofs}\label{app:proofs}

\subsection{Proof of \texorpdfstring{\Cref{thm:bias}}{Theorem 3.1}}\label{app:bias}

Fix $q$ with $\Delta_q > 0$. Write $D_k := \hat\pi_N(k) - \pi_{q,k}$ for
the centered empirical PMF on $N$, and similarly $D'_k$ for $E$.

\paragraph{Hoeffding bound.}
Each $\hat\pi_N(k) - \pi_{q,k}$ is a sample mean of $[0,1]$-bounded
i.i.d.\ Bernoulli variables, hence $1/(2\sqrt n)$-sub-Gaussian by
Hoeffding's lemma. The standard sub-Gaussian maximum bound gives
\begin{equation}\label{eq:max-subg}
\mathbb{E}\!\big[\max_k |D_k|\big]
\;\le\; \frac{1}{2\sqrt n}\sqrt{2\log(2K_q)}
\;=\; \sqrt{\tfrac{\log(2K_q)}{2n}}.
\end{equation}

\emph{For $i = 1$:} the map $(x_1,\dots,x_K) \mapsto \max_k x_k$ is
$1$-Lipschitz in $\ell_\infty$, so
$|\hat c_1 - c^\star_q| \le \max_k |D_k|$. Taking expectations and using
\eqref{eq:max-subg} bounds $|\mathbb{E}[\hat c_1 - c^\star_q \mid q]|$ by
$\sqrt{\log(2K_q)/(2n)}$.

\emph{For $i = 2$:} by conditional unbiasedness
\eqref{eq:c2-cond-unbiased},
$\mathbb{E}[\hat c_2 \mid q] = \mathbb{E}[\pi_{q,\hat z_N} \mid q]$. Since
$\hat\pi_N(\hat z_N) \ge \hat\pi_N(z^\star_q)$ by definition of $\hat z_N$,
\[
\pi_{q,\hat z_N} \ge \hat\pi_N(\hat z_N) - \max_k|D_k|
\ge \hat\pi_N(z^\star_q) - \max_k|D_k|
\ge p_q - 2\max_k|D_k|,
\]
so $0 \le c^\star_q - \pi_{q,\hat z_N} \le 2\max_k|D_k|$. Hence
$|\mathbb{E}[\hat c_2 - c^\star_q \mid q]| \le 2\sqrt{\log(2K_q)/(2n)}$,
which absorbs into the same Hoeffding rate up to a universal constant.

\paragraph{Bernstein bound.}
For both $i \in \{1,2\}$, we show
$|\mathbb{E}[\hat c_i - c^\star_q \mid q]| \le \Pr(\hat z_N \ne z^\star_q \mid q)$.

\emph{For $i = 1$:} the increment representation
$\hat c_1 = \hat\pi_N(z^\star_q) + (\hat\pi_N(\hat z_N) - \hat\pi_N(z^\star_q))$
combined with $\mathbb{E}[\hat\pi_N(z^\star_q)] = c^\star_q$ gives
\[
0 \le \mathbb{E}[\hat c_1 - c^\star_q \mid q]
= \mathbb{E}\!\big[(\hat\pi_N(\hat z_N) - \hat\pi_N(z^\star_q))\,\mathbf{1}_{\hat z_N \ne z^\star_q}\big]
\le \Pr(\hat z_N \ne z^\star_q \mid q),
\]
since the increment lies in $[0, 1]$.

\emph{For $i = 2$:}
$\mathbb{E}[\hat c_2 - c^\star_q \mid q] = -\mathbb{E}[(c^\star_q - \pi_{q,\hat z_N})\mathbf{1}_{\hat z_N \ne z^\star_q}]$
and $0 \le c^\star_q - \pi_{q,\hat z_N} \le 1$, so the same bound applies.

By union bound,
$\Pr(\hat z_N \ne z^\star_q \mid q) \le \sum_{k \ne z^\star_q}\Pr(\hat\pi_N(k) \ge \hat\pi_N(z^\star_q) \mid q)$.
For each $k \ne z^\star_q$, define
$\xi^{(k)}_i := \mathbf{1}_{Z_i = z^\star_q} - \mathbf{1}_{Z_i = k}$, so
$\mathbb{E}[\xi^{(k)}] = \pi_{q,z^\star_q} - \pi_{q,k} \ge \Delta_q$,
$|\xi^{(k)}| \le 1$, and
$\mathrm{Var}(\xi^{(k)}) \le p_q + \pi_{q,k} \le 2p_q$. Bernstein's inequality
yields
\[
\Pr\!\Big(\tfrac{1}{n}\sum_{i=1}^n \xi^{(k)}_i \le 0 \,\Big|\, q\Big)
\;\le\; \exp\!\left(-\frac{n\Delta_q^2/2}{2p_q + \Delta_q/3}\right)
\;\le\; \exp\!\left(-\frac{n\Delta_q^2}{4p_q + 1}\right).
\]
Summing over the $K_q-1$ runners-up gives the Bernstein-type bound, which
(after absorbing the small constant $\Delta_q/3 \le 1/3$ in the
denominator) is dominated by $(K_q-1)\exp(-n\Delta_q^2/(2p_q))$ up to
universal constants. Combining with the Hoeffding bound proves
\eqref{eq:confidence-bound}. \qed

\subsection{Proof of Theorem~\ref{thm:ece-bound}}
\begin{proof}
Using the equivalent representation of fixed-bin ECE,
\[
\operatorname{ECE}(c,a)
=
\sum_{\ell=1}^L
\left|
\mathbb E\left[
(a-c)\mathbf 1\{C\in\mathcal I_\ell\}
\right]
\right|,
\]
the reverse triangle inequality gives
\[
\begin{aligned}
&
\left|
\operatorname{ECE}(\hat c,\hat a)
-
\operatorname{ECE}(c^\star,a^\star)
\right| \\
&\le
\sum_{\ell=1}^L
\mathbb E\left|
(\hat a-\hat c)\mathbf 1\{\hat c\in\mathcal I_\ell\}
-
(a^\star-c^\star)\mathbf 1\{c^\star\in\mathcal I_\ell\}
\right|.
\end{aligned}
\]
Split according to whether $\hat c$ and $c^\star$ fall in the same bin. On the event
$\operatorname{bin}(\hat c)=\operatorname{bin}(c^\star)$, only one bin contributes, and
\[
|(\hat a-\hat c)-(a^\star-c^\star)|
\le
|\hat a-a^\star|+|\hat c-c^\star|.
\]
On the event $\operatorname{bin}(\hat c)\neq\operatorname{bin}(c^\star)$, at most two bins
contribute, and each contribution is bounded by one. Therefore
\[
\left|
\operatorname{ECE}(\hat c,\hat a)
-
\operatorname{ECE}(c^\star,a^\star)
\right|
\le
\mathbb E|\hat a-a^\star|
+
\mathbb E|\hat c-c^\star|
+
2\mathbb P\{\operatorname{bin}(\hat c)\neq\operatorname{bin}(c^\star)\}.
\]
Since $a_n,a^\star\in\{0,1\}$,
\[
\mathbb E|\hat a-a^\star|=\mathbb P(\hat a\neq a^\star)=\delta_n,
\]
and by definition
\[
\mathbb E|\hat c-c^\star|=\varepsilon_n.
\]
It remains to bound the bin-crossing probability. If
$\operatorname{bin}(\hat c)\neq\operatorname{bin}(c^\star)$, then either
$|\hat c-c^\star|>\eta$ or $c^\star$ lies within distance $\eta$ of a bin boundary.
Hence
\[
\mathbb P\{\operatorname{bin}(\hat c)\neq\operatorname{bin}(c^\star)\}
\le
\mathbb P\{|\hat c-c^\star|>\eta\}
+
\mathbb P\{\operatorname{dist}(c^\star,\mathcal T)\le \eta\}.
\]
By Markov's inequality,
\[
\mathbb P\{|\hat c-c^\star|>\eta\}
\le
\frac{\varepsilon_n}{\eta}.
\]
This proves the first claim.

If $c^\star$ has density bounded by $M$, then the $\eta$-neighborhood of the
$L-1$ bin boundaries has total length at most $2(L-1)\eta$, so
\[
\mathbb P\{\operatorname{dist}(c^\star,\mathcal T)\le \eta\}
\le
2M(L-1)\eta.
\]
Therefore
\[
\left|
\operatorname{ECE}(\hat c,\hat a)
-
\operatorname{ECE}(c^\star,a^\star)
\right|
\le
\delta_n+\varepsilon_n
+
2\left\{
\frac{\varepsilon_n}{\eta}
+
2M(L-1)\eta
\right\}.
\]
Optimizing over $\eta$ gives
\[
2\left\{
\frac{\varepsilon_n}{\eta}
+
2M(L-1)\eta
\right\}
\le
4\sqrt{2M(L-1)\varepsilon_n}.
\]
This completes the proof.
\end{proof}
\subsection{Proof of Theorem~\ref{boundc1}}
Before we provide the proof of Theorem \ref{boundc1}, we first record a standard pairwise comparison bound.  Throughout this
section, fix a question $q$ and, without loss of generality, relabel the
population semantic mode as
\[
z_q^\star=1,\qquad c_q^\star=\pi_{q,1}.
\]
Let
\[
\pi_{q,(2)}:=\max_{k\neq 1}\pi_{q,k},\qquad
\Delta_q:=\pi_{q,1}-\pi_{q,(2)},\qquad
p_q:=\pi_{q,1}+\pi_{q,(2)}.
\]
We assume $\Delta_q>0$ for $Q$-almost every $q$.

\begin{lemma}[Selection error]
For every fixed $q$ with $\Delta_q>0$,
\[
\mathbb P(\hat z_N\neq z_q^\star\mid q)
\le
(K_q-1)\exp\left(-\frac{n\Delta_q^2}{2p_q}\right).
\]
\end{lemma}

\begin{proof}
For any $k\neq 1$, the event $\hat z_N=k$ implies
\[
\hat\pi_N(k)\geq \hat\pi_N(1).
\]
Equivalently,
\[
\frac1n\sum_{i\in N}\xi_i^{(k)}\leq 0,
\qquad
\xi_i^{(k)}:=\mathbf 1\{Z_i=1\}-\mathbf 1\{Z_i=k\}.
\]
Now
\[
\mathbb E[\xi_i^{(k)}\mid q]
=
\pi_{q,1}-\pi_{q,k}
\geq \Delta_q,
\]
and the random variable $\xi_i^{(k)}$ is supported on $\{-1,0,1\}$ with
variance controlled by the total top-versus-$k$ probability mass:
\[
\operatorname{Var}(\xi_i^{(k)}\mid q)
\leq
\mathbb E[(\xi_i^{(k)})^2\mid q]
=
\pi_{q,1}+\pi_{q,k}
\leq p_q .
\]
Applying the Bernstein--Chernoff bound for this three-valued comparison
statistic gives
\[
\mathbb P\left(
\hat\pi_N(k)\geq \hat\pi_N(1)\mid q
\right)
\leq
\exp\left(-\frac{n\Delta_q^2}{2p_q}\right).
\]
Taking a union bound over all $k\neq 1$ yields
\[
\mathbb P(\hat z_N\neq z_q^\star\mid q)
\leq
\sum_{k\neq 1}
\mathbb P\left(
\hat\pi_N(k)\geq \hat\pi_N(1)\mid q
\right)
\leq
(K_q-1)\exp\left(-\frac{n\Delta_q^2}{2p_q}\right).
\]
\end{proof}

\begin{proof}[Proof of Theorem~\ref{boundc1}]
Let
\[
\hat c=\hat c_1=\max_{k\in\mathcal Z_q}\hat\pi_N(k),
\qquad
\hat a=Y_q(\hat z_N),
\qquad
a_q^\star=Y_q(z_q^\star).
\]
We first bound the confidence error
\[
\varepsilon_n=\mathbb E|\hat c_1-c_q^\star|.
\]

For fixed $q$, write
\[
D_k:=\hat\pi_N(k)-\pi_{q,k}.
\]
Since the maximum map is $1$-Lipschitz with respect to the $\ell_\infty$
norm,
\[
|\hat c_1-c_q^\star|
=
\left|
\max_k \hat\pi_N(k)-\max_k\pi_{q,k}
\right|
\leq
\max_k |D_k|.
\]
Each $D_k$ is the centered average of $n$ Bernoulli variables and is
$1/(2\sqrt n)$-sub-Gaussian. Therefore,
\[
\mathbb E\left[\max_k |D_k|\mid q\right]
\leq
\sqrt{\frac{\log(2K_q)}{2n}}.
\]
Thus
\[
\mathbb E\left[|\hat c_1-c_q^\star|\mid q\right]
\leq
\sqrt{\frac{\log(2K_q)}{2n}}.
\]

We next derive a margin-sensitive bound.  On the event
$\{\hat z_N=z_q^\star\}$, we have
\[
\hat c_1=\hat\pi_N(z_q^\star)=\hat\pi_N(1),
\]
and hence
\[
|\hat c_1-c_q^\star|
=
|\hat\pi_N(1)-\pi_{q,1}|.
\]
On the complement $\{\hat z_N\neq z_q^\star\}$, the trivial bound
$|\hat c_1-c_q^\star|\leq 1$ gives
\[
\mathbb E\left[|\hat c_1-c_q^\star|\mid q\right]
\leq
\mathbb E\left[|\hat\pi_N(1)-\pi_{q,1}|\mid q\right]
+
\mathbb P(\hat z_N\neq z_q^\star\mid q).
\]
Since $\hat\pi_N(1)$ is the average of $n$ Bernoulli variables with
success probability $\pi_{q,1}$,
\[
\mathbb E\left[|\hat\pi_N(1)-\pi_{q,1}|\mid q\right]
\leq
\sqrt{\operatorname{Var}(\hat\pi_N(1)\mid q)}
=
\sqrt{\frac{\pi_{q,1}(1-\pi_{q,1})}{n}}.
\]
Combining this with the selection-error lemma yields
\[
\mathbb E\left[|\hat c_1-c_q^\star|\mid q\right]
\leq
\sqrt{\frac{\pi_{q,1}(1-\pi_{q,1})}{n}}
+
(K_q-1)\exp\left(-\frac{n\Delta_q^2}{2p_q}\right).
\]
Together with the uniform Hoeffding bound, we obtain
\[
\mathbb E\left[|\hat c_1-c_q^\star|\mid q\right]
\leq
\min\left\{
\sqrt{\frac{\log(2K_q)}{2n}},
\,
\sqrt{\frac{\pi_{q,1}(1-\pi_{q,1})}{n}}
+
(K_q-1)\exp\left(-\frac{n\Delta_q^2}{2p_q}\right)
\right\}.
\]
Taking expectation over $q\sim Q$ gives
\[
\varepsilon_n
\leq
\mathbb E_q
\left[
\min\left\{
\sqrt{\frac{\log(2K_q)}{2n}},
\,
\sqrt{\frac{\pi_{q,1}(1-\pi_{q,1})}{n}}
+
(K_q-1)\exp\left(-\frac{n\Delta_q^2}{2p_q}\right)
\right\}
\right].
\]

It remains to bound the selection error
\[
\delta_n=\mathbb P(\hat a\neq a_q^\star).
\]
Since $\hat a=Y_q(\hat z_N)$ and $a_q^\star=Y_q(z_q^\star)$,
\[
\{\hat a\neq a_q^\star\}
\subseteq
\{\hat z_N\neq z_q^\star\}.
\]
Therefore, by the selection-error lemma,
\[
\mathbb P(\hat a\neq a_q^\star\mid q)
\leq
(K_q-1)\exp\left(-\frac{n\Delta_q^2}{2p_q}\right).
\]
Taking expectation over $q\sim Q$ gives
\[
\delta_n
\leq
\mathbb E_q
\left[
(K_q-1)\exp\left(-\frac{n\Delta_q^2}{2p_q}\right)
\right].
\]
This proves Theorem~\ref{boundc1}.
\end{proof}
\subsection{Proof of Theorem~\ref{thm:bound-c2}}
\begin{proof}
Let
\[
\hat c=\hat c_2=\hat\pi_E(\hat z_N),
\qquad
\hat a=Y_q(\hat z_N),
\qquad
a_q^\star=Y_q(z_q^\star).
\]
Again relabel $z_q^\star=1$, so that $c_q^\star=\pi_{q,1}$.

For fixed $q$, decompose
\[
|\hat c_2-c_q^\star|
=
|\hat\pi_E(\hat z_N)-\pi_{q,1}|
\leq
|\hat\pi_E(\hat z_N)-\pi_{q,\hat z_N}|
+
|\pi_{q,\hat z_N}-\pi_{q,1}|.
\]
We control the two terms separately.

For the first term, condition on $q$ and $\hat z_N$.  Since the evaluation
block $E$ is independent of the selection block $N$,
\[
\hat\pi_E(\hat z_N)
=
\frac1m\sum_{i\in E}\mathbf 1\{Z_i=\hat z_N\}
\]
is, conditionally on $(q,\hat z_N)$, the average of $m$ Bernoulli variables
with success probability $\pi_{q,\hat z_N}$. Hence
\[
\mathbb E\left[
|\hat\pi_E(\hat z_N)-\pi_{q,\hat z_N}|
\mid q,\hat z_N
\right]
\leq
\sqrt{\frac{\pi_{q,\hat z_N}(1-\pi_{q,\hat z_N})}{m}}.
\]
For every $k$, one has
\[
\pi_{q,k}(1-\pi_{q,k})
\leq
\pi_{q,1}(1-\pi_{q,1}).
\]
Indeed, this is immediate for $k=1$.  For $k\neq 1$, since
$\pi_{q,k}\leq \pi_{q,1}$ and $\pi_{q,k}\leq 1-\pi_{q,1}$, the Bernoulli
variance of class $k$ is no larger than that of the modal class.  Therefore,
\[
\mathbb E\left[
|\hat\pi_E(\hat z_N)-\pi_{q,\hat z_N}|
\mid q
\right]
\leq
\sqrt{\frac{\pi_{q,1}(1-\pi_{q,1})}{m}}.
\]

For the second term,
\[
|\pi_{q,\hat z_N}-\pi_{q,1}|
=
0
\quad\text{on } \{\hat z_N=z_q^\star\},
\]
and it is at most $1$ otherwise. Hence
\[
\mathbb E\left[
|\pi_{q,\hat z_N}-\pi_{q,1}|
\mid q
\right]
\leq
\mathbb P(\hat z_N\neq z_q^\star\mid q).
\]
Using the selection-error lemma,
\[
\mathbb E\left[
|\pi_{q,\hat z_N}-\pi_{q,1}|
\mid q
\right]
\leq
(K_q-1)\exp\left(-\frac{n\Delta_q^2}{2p_q}\right).
\]
Combining the two bounds gives
\[
\mathbb E\left[|\hat c_2-c_q^\star|\mid q\right]
\leq
\sqrt{\frac{\pi_{q,1}(1-\pi_{q,1})}{m}}
+
(K_q-1)\exp\left(-\frac{n\Delta_q^2}{2p_q}\right).
\]
Taking expectation over $q\sim Q$ yields
\[
\varepsilon_n
\leq
\mathbb E_q
\left[
\sqrt{\frac{\pi_{q,1}(1-\pi_{q,1})}{m}}
+
(K_q-1)\exp\left(-\frac{n\Delta_q^2}{2p_q}\right)
\right].
\]

The selection error is exactly the same as in Theorem 5.3, because Sem1 and
Sem2 use the same selected answer $\hat z_N$.  Namely,
\[
\{\hat a\neq a_q^\star\}
\subseteq
\{\hat z_N\neq z_q^\star\},
\]
so
\[
\delta_n
\leq
\mathbb E_q
\left[
(K_q-1)\exp\left(-\frac{n\Delta_q^2}{2p_q}\right)
\right].
\]
This completes the proof.
\end{proof}
\subsection{Proof of the bias expansions \eqref{eq:bias-expansion}}\label{app:bias-expansion}

Fix $q \in \mathcal{Q}_{\mathrm{low}}$ and let
$z^{(2)}_q := \argmax_{k \ne z^\star_q}\pi_{q,k}$ denote the runner-up
class. Throughout we work in the local regime $\Delta_q = O(n^{-1/2})$ so
that $\tilde m_q$ is bounded.

\paragraph{Reduction to the top-two classes.}
Let $\mathcal{E}_q := \{\hat z_N \in \{z^\star_q, z^{(2)}_q\}\}$. By the
Bernstein argument in \Cref{app:bias} applied to each runner-up $j \notin
\{z^\star_q, z^{(2)}_q\}$ (which has gap $\Delta_q^{(j)} \ge \Delta_q$),
\[
\Pr(\mathcal{E}_q^c \mid q) \;\le\; (K_q-2)\exp\!\big(-c\,n\Delta_q^2/p_q\big).
\]
In the regime $\tilde m_q^2 < \log K_q$ this is
$o(n^{-1/2})$ uniformly, and contributes only $o(n^{-1/2})$ to
$\mathbb{E}[\hat c_i - c^\star_q \mid q]$ since
$|\hat c_i - c^\star_q| \le 1$. We may therefore work conditionally on
$\mathcal{E}_q$.

\paragraph{Local CLT on the boundary statistic.}
Define
$V := \hat\pi_N(z^\star_q) - \hat\pi_N(z^{(2)}_q)$. By the bivariate
multinomial CLT,
\[
\sqrt n\,(V - \Delta_q) \;\Rightarrow\; \mathcal{N}(0, \tau_q^2),
\qquad
\tau_q^2 := p_q + p^{(2)}_q - \Delta_q^2 \;=\; p_q\,(1 + o(1))
\]
in low-margin (under the usual normalization where $\sqrt p_q$ in
\eqref{eq:bias-expansion} denotes the per-class boundary scale). Standardize:
$V = \sqrt{p_q/n}(\,\tilde m_q + Y_n)$ with $Y_n \Rightarrow \mathcal{N}(0,1)$
uniformly.

\paragraph{Sem$_1$ bias (Jensen gap).}
On $\mathcal{E}_q$, $\hat c_1 = \max(\hat\pi_N(z^\star_q), \hat\pi_N(z^{(2)}_q))$, hence
\[
\hat c_1 - \hat\pi_N(z^\star_q) \;=\; V^- \;=\; \max(-V, 0).
\]
Since $\mathbb{E}[\hat\pi_N(z^\star_q) \mid q] = c^\star_q$,
\[
\mathbb{E}[\hat c_1 - c^\star_q \mid q]
\;=\; \mathbb{E}[V^-]
\;=\; \sqrt{p_q/n}\;\mathbb{E}\!\big[(\tilde m_q + Y_n)^-\big]
\;\to\; \sqrt{p_q/n}\;\mathbb{E}\!\big[(Y - \tilde m_q)^+\big],
\]
using the symmetry $-Y \stackrel{d}{=} Y$ for $Y \sim \mathcal{N}(0,1)$.
Applying the standard formula $\mathbb{E}[(Y - \mu)^+] = \varphi(\mu) - \mu\Phi(-\mu)$,
\[
\mathbb{E}[\hat c_1 - c^\star_q \mid q]
\;=\; \tfrac{1}{\sqrt n}\sqrt{p_q}\;\big[\varphi(\tilde m_q) - \tilde m_q\Phi(-\tilde m_q)\big] + o(n^{-1/2})
\;=\; \tfrac{1}{\sqrt n}\sqrt{p_q}\,J(\tilde\lambda_q) + o(n^{-1/2}),
\]
since $\tilde m_q = 2\tilde\lambda_q$ and
$J(\tilde\lambda) = \varphi(2\tilde\lambda) - 2\tilde\lambda\Phi(-2\tilde\lambda)$.

\paragraph{Sem$_2$ bias (selection gap).}
By conditional unbiasedness \eqref{eq:c2-cond-unbiased},
$\mathbb{E}[\hat c_2 \mid q, \hat z_N] = \pi_{q,\hat z_N}$. On $\mathcal{E}_q$,
$\pi_{q,\hat z_N}$ takes value $c^\star_q$ on $\{\hat z_N = z^\star_q\}$ and
value $c^\star_q - \Delta_q$ on $\{\hat z_N = z^{(2)}_q\}$. Therefore
\[
\mathbb{E}[\hat c_2 - c^\star_q \mid q, \mathcal{E}_q]
\;=\; -\Delta_q\,\Pr(\hat z_N = z^{(2)}_q \mid q, \mathcal{E}_q)
\;=\; -\Delta_q\,\Pr(V \le 0 \mid q, \mathcal{E}_q).
\]
By the local CLT, $\Pr(V \le 0) \to \Phi(-\tilde m_q)$, and substituting
$\Delta_q = \sqrt{p_q/n}\,\tilde m_q$,
\[
\mathbb{E}[\hat c_2 - c^\star_q \mid q]
\;=\; -\tfrac{1}{\sqrt n}\sqrt{p_q}\,\tilde m_q\Phi(-\tilde m_q) + o(n^{-1/2})
\;=\; -\tfrac{1}{\sqrt n}\sqrt{p_q}\,S(\tilde\lambda_q) + o(n^{-1/2}),
\]
since $S(\tilde\lambda) = 2\tilde\lambda\Phi(-2\tilde\lambda) = \tilde m_q\Phi(-\tilde m_q)$.
\qed

\subsection{Proof of \texorpdfstring{\Cref{thm:gap}}{Corollary 3.3}}\label{app:gap}

Let $A := \mathbb{E}_q[c^\star_q] - \bar{a} > 0$ denote the population over-confidence
gap. By the bias expansions \eqref{eq:bias-expansion},
\[
\mathbb{E}_q[\hat c_1] - \bar{a}
= A + \tfrac{1}{\sqrt n}\,\mathbb{E}_q[\sqrt{p_q}\,J(\tilde\lambda_q)] + o(n^{-1/2}),
\]
\[
\mathbb{E}_q[\hat c_2] - \bar{a}
= A - \tfrac{1}{\sqrt n}\,\mathbb{E}_q[\sqrt{p_q}\,S(\tilde\lambda_q)] + o(n^{-1/2}).
\]
Since $J, S \ge 0$ and $A > 0$, both expressions are positive for $n$
sufficiently large, so the absolute values open in the positive direction:
$\mathrm{ECE}(\hat c_1) = \mathbb{E}_q[\hat c_1] - \bar{a}$ and
$\mathrm{ECE}(\hat c_2) = \mathbb{E}_q[\hat c_2] - \bar{a}$ to leading order.
Subtracting,
\[
\mathrm{ECE}(\hat c_1) - \mathrm{ECE}(\hat c_2)
= \tfrac{1}{\sqrt n}\,\mathbb{E}_q\!\big[\sqrt{p_q}\,(J(\tilde\lambda_q) + S(\tilde\lambda_q))\big] + o(n^{-1/2})
= \tfrac{1}{\sqrt n}\,\mathbb{E}_q[\sqrt{p_q}\,g_A(\tilde\lambda_q)] + o(n^{-1/2}),
\]
which is strictly positive since $g_A(\tilde\lambda) = \varphi(2\tilde\lambda) > 0$
on all of $\mathcal{Q}_{\mathrm{low}}$. \qed

\subsection{Proof of \texorpdfstring{\Cref{thm:sharp}}{Theorem 3.4}}\label{app:sharp}

We compare $|\mathbb{E}_q[\hat c_i] - \mathbb{E}_q[c^\star_q]|$ for $i \in \{1,2\}$
under non-degenerate over-confidence $A := \mathbb{E}_q[c^\star_q] - a \ge c_0/\sqrt n$
(with $c_0 > 0$ fixed) and population support in $\mathcal{Q}_{\mathrm{JDR}}$.

By \eqref{eq:bias-expansion},
\[
\mathbb{E}_q[\hat c_i] - \mathbb{E}_q[c^\star_q]
= \tfrac{(-1)^{i+1}}{\sqrt n}\,\mathbb{E}_q[\sqrt{p_q}\,h_i(\tilde\lambda_q)] + o(n^{-1/2}),
\]
with $h_1 = J$ and $h_2 = S$. Both are $O(n^{-1/2})$, while the
non-degeneracy hypothesis ensures the comparison \emph{against the oracle}
$\mathbb{E}_q[c^\star_q]$ does not flip sign within the $o(n^{-1/2})$ remainder
(this enters when forming
$|\mathbb{E}_q[\hat c_i] - \mathbb{E}_q[c^\star_q]|$ in conjunction with
$\mathrm{ECE}^\star = A$).

Specifically, the absolute oracle distance satisfies
\[
\big|\mathbb{E}_q[\hat c_1] - \mathbb{E}_q[c^\star_q]\big|
= \tfrac{1}{\sqrt n}\,\mathbb{E}_q[\sqrt{p_q}\,J(\tilde\lambda_q)] + o(n^{-1/2}),
\]
\[
\big|\mathbb{E}_q[\hat c_2] - \mathbb{E}_q[c^\star_q]\big|
= \tfrac{1}{\sqrt n}\,\mathbb{E}_q[\sqrt{p_q}\,S(\tilde\lambda_q)] + o(n^{-1/2}),
\]
and subtracting,
\[
\big|\mathbb{E}_q[\hat c_1] - \mathbb{E}_q[c^\star_q]\big|
- \big|\mathbb{E}_q[\hat c_2] - \mathbb{E}_q[c^\star_q]\big|
= \tfrac{1}{\sqrt n}\,\mathbb{E}_q\!\big[\sqrt{p_q}\,(J(\tilde\lambda_q) - S(\tilde\lambda_q))\big] + o(n^{-1/2}).
\]
By definition $g_B = J - S$. On $\mathcal{Q}_{\mathrm{JDR}}$,
$\tilde\lambda_q < \tilde\lambda^\star$, so $g_B(\tilde\lambda_q) > 0$
\eqref{app:lambda-star}. The leading term is therefore strictly positive,
proving \eqref{eq:sharp}. \qed

\subsection{Uniqueness of \(\widetilde\lambda^\star\) via Mills ratio}
\label{app:lambda-star}
We prove that
\[
g_B(\widetilde\lambda)
:=
\phi(2\widetilde\lambda)
-
4\widetilde\lambda\Phi(-2\widetilde\lambda)
\]
has a unique root on \((0,\infty)\). Substitute
\(u:=2\widetilde\lambda\). Then
\[
g_B=0
\Longleftrightarrow
\phi(u)=2u\Phi(-u)
\Longleftrightarrow
r(u)=2u,
\qquad
r(u):=\frac{\phi(u)}{\Phi(-u)}.
\]
Equivalently,
\[
\frac{u}{r(u)}=\frac12.
\]

Let
\[
h(u):=r(u)-2u.
\]
We have \(h(0)=r(0)>0\). By the Mills-ratio asymptotic
\(r(u)=u+u^{-1}+o(u^{-1})\), \(h(u)\to-\infty\) as \(u\to\infty\).
Moreover, the inverse-Mills derivative identity gives
\[
r'(u)=r(u)(r(u)-u),
\]
and the standard bound \(0<r'(u)<1\) for \(u>0\) implies
\[
h'(u)=r'(u)-2<0.
\]
Thus \(h\) is strictly decreasing and has a unique positive root.
Numerically, this root is \(u^\star\approx0.6125\), so
\[
\widetilde\lambda^\star=u^\star/2\approx0.306.
\]

\section{Algorithm}
\label{app:algorithm}

\Cref{alg:semece} summarizes the Sem-ECE framework on a single
question $q$, paralleling the definitions in \Cref{sec:framework}.

\begin{algorithm}[h]
\caption{The Sem-ECE framework on a question \(q\).}
\label{alg:semece}
\begin{algorithmic}[1]
\Require Question \(q\); selection size \(n\); evaluation size \(m\);
semantic clustering oracle \(\mathrm{Cluster}(\cdot)\).
\State Generate \(n+m\) i.i.d.\ responses
\(A_1,\ldots,A_{n+m}\) from the LLM on \(q\).
\State \((Z_1,\ldots,Z_{n+m}) \gets
       \mathrm{Cluster}(A_1,\ldots,A_{n+m})\). \Comment{semantic class
       labels}
\State Partition \([n+m] = N \sqcup E\) with \(|N|=n\), \(|E|=m\).
\State \(\hat z_N \gets
       \argmax_{k\in\mathcal Z_q}\hat\pi_N(k)\).
       \Comment{deployed answer}
\State \(\hat c_1 \gets
       \max_{k\in\mathcal Z_q}\hat\pi_N(k)\).
       \Comment{same-sample confidence \eqref{eq:c1-def}}
\State \(\hat c_2 \gets
       \hat\pi_E(\hat z_N)\).
       \Comment{held-out confidence \eqref{eq:c2-def}}
\State \Return \((\hat z_N,\,\hat c_1,\,\hat c_2)\).
\end{algorithmic}
\end{algorithm}

In experiments (\Cref{sec:experiments}), $\hat c_2$ is computed by
averaging $\hat\pi_E(\hat z_N)$ over $R$ random half-splits of a
pooled sample of size $n + m$ to recycle samples across the two
estimators.

\section{Additional Experimental Figures}
\label{app:additional-figures}

This section reports the full diagnostic figures summarized in the main text.

\subsection{Margin-stratified ECE}

\begin{figure}[p]
\centering
\includegraphics[width=0.95\linewidth]{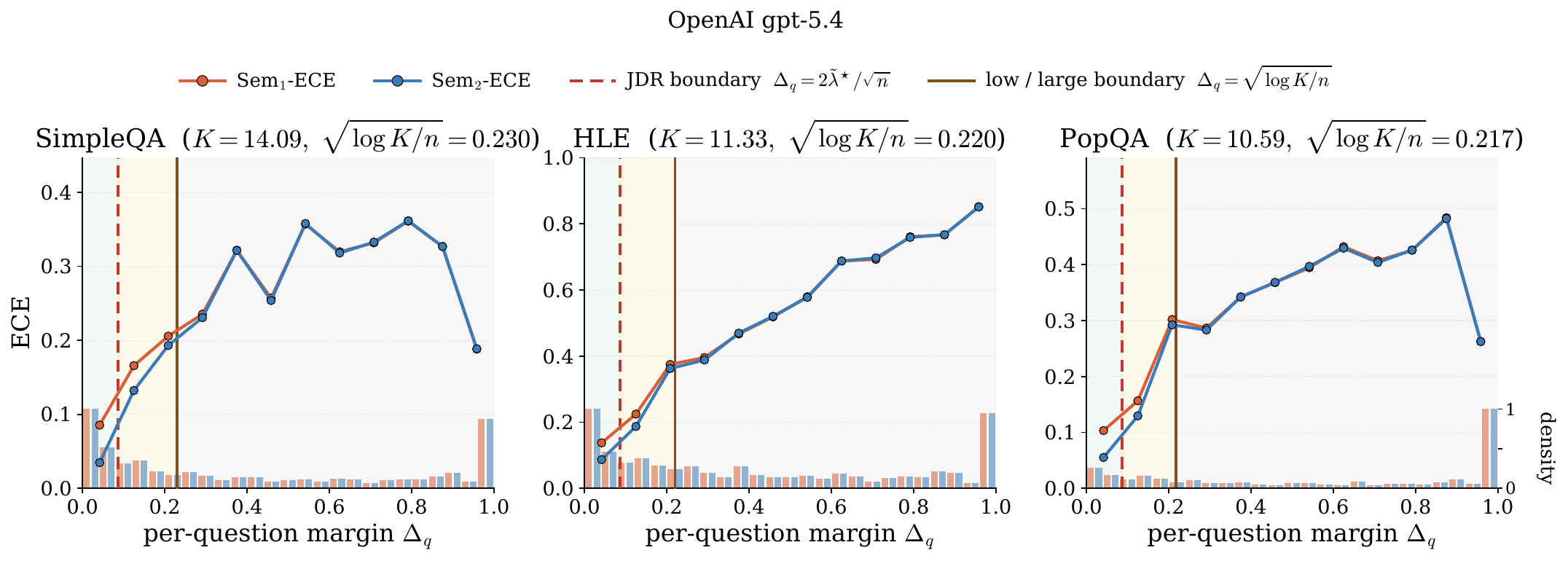}
\caption{
Margin-stratified ECE curves for OpenAI on SimpleQA, HLE, and PopQA.
}
\label{fig:app-ece-margin-openai}
\end{figure}

\begin{figure}[p]
\centering
\includegraphics[width=0.95\linewidth]{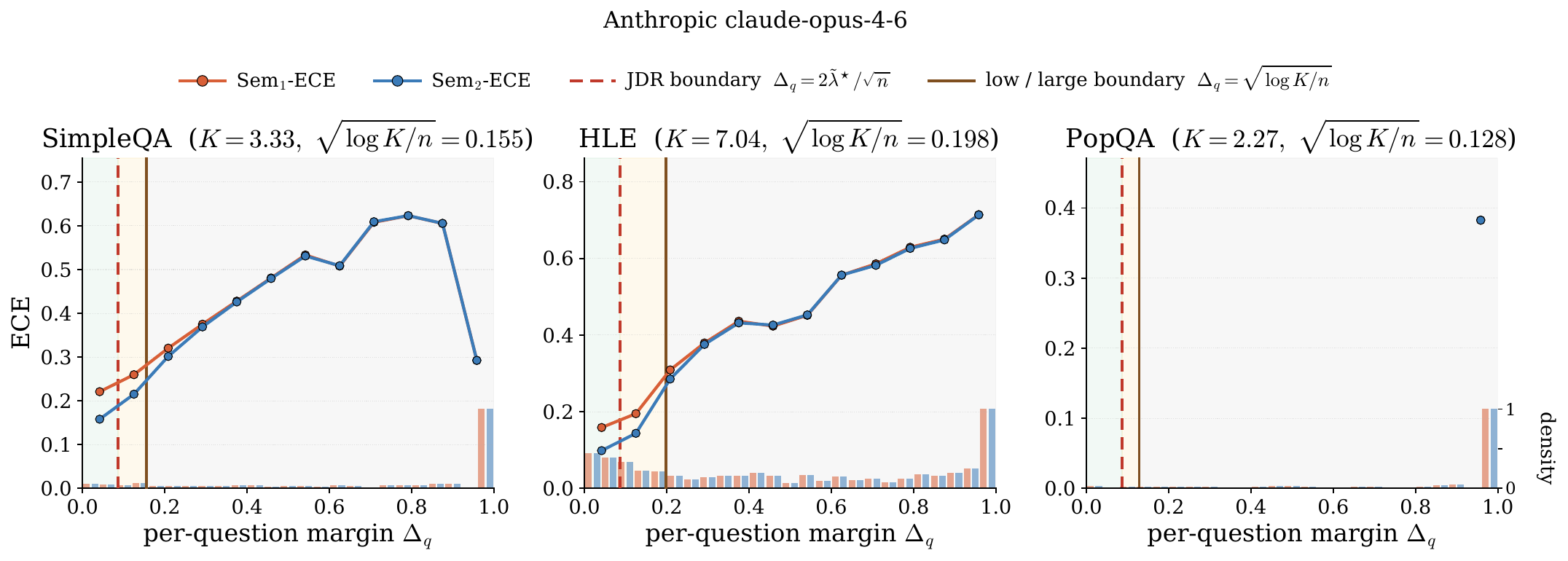}
\caption{
Margin-stratified ECE curves for Anthropic on SimpleQA, HLE, and PopQA.
}
\label{fig:app-ece-margin-anthropic}
\end{figure}

\begin{figure}[p]
\centering
\includegraphics[width=0.95\linewidth]{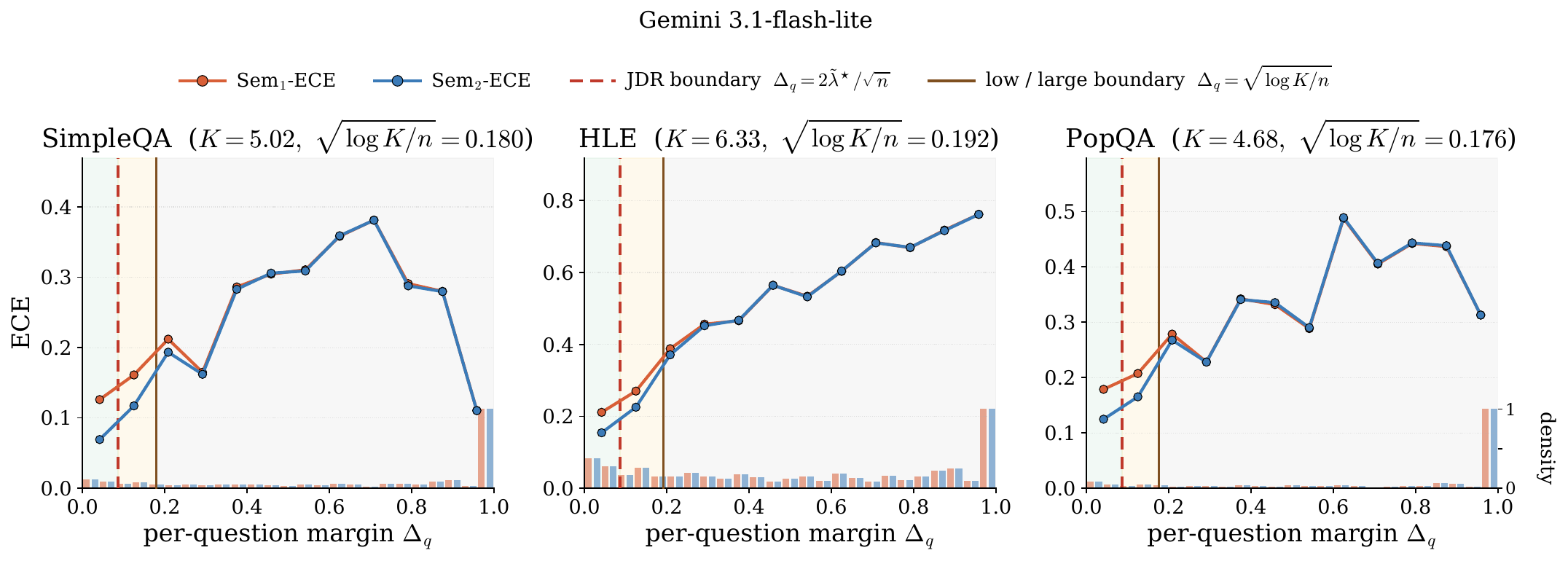}
\caption{
Margin-stratified ECE curves for Gemini on SimpleQA, HLE, and PopQA.
}
\label{fig:app-ece-margin-gemini}
\end{figure}

\begin{figure}[p]
\centering
\includegraphics[width=0.95\linewidth]{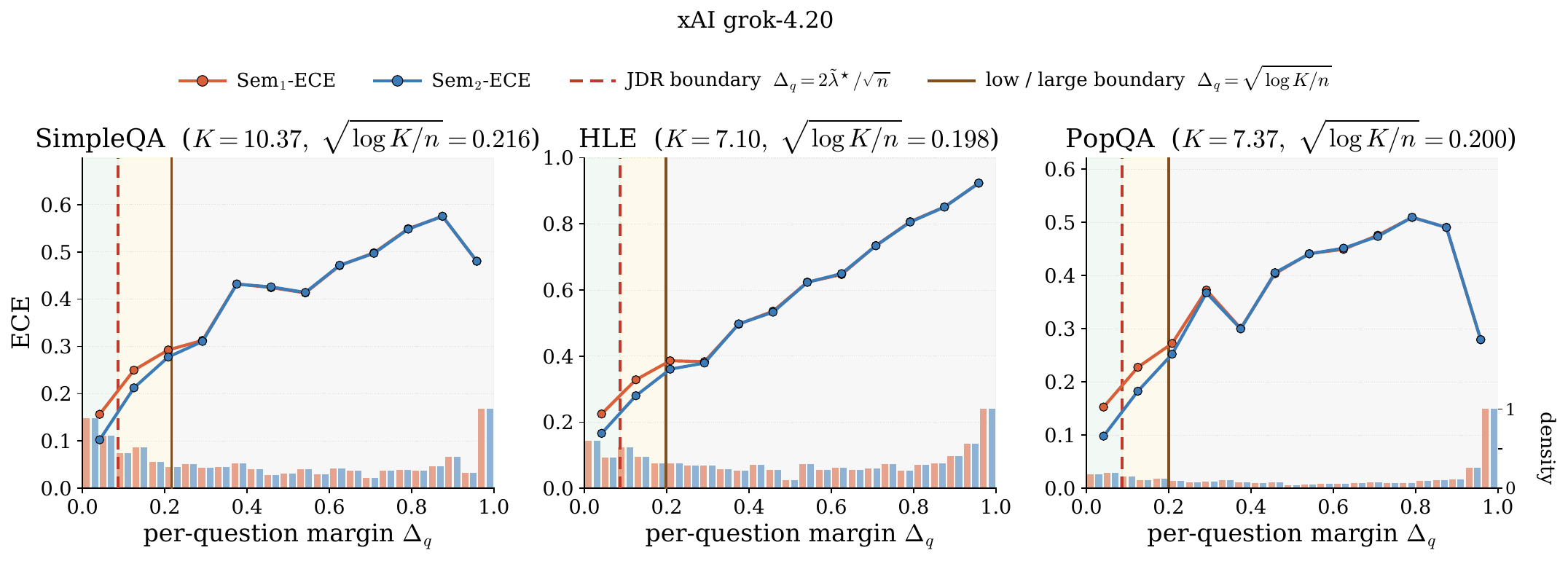}
\caption{
Margin-stratified ECE curves for xAI on SimpleQA, HLE, and PopQA.
}
\label{fig:app-ece-margin-xai}
\end{figure}

\begin{figure}[p]
\centering
\includegraphics[width=0.95\linewidth]{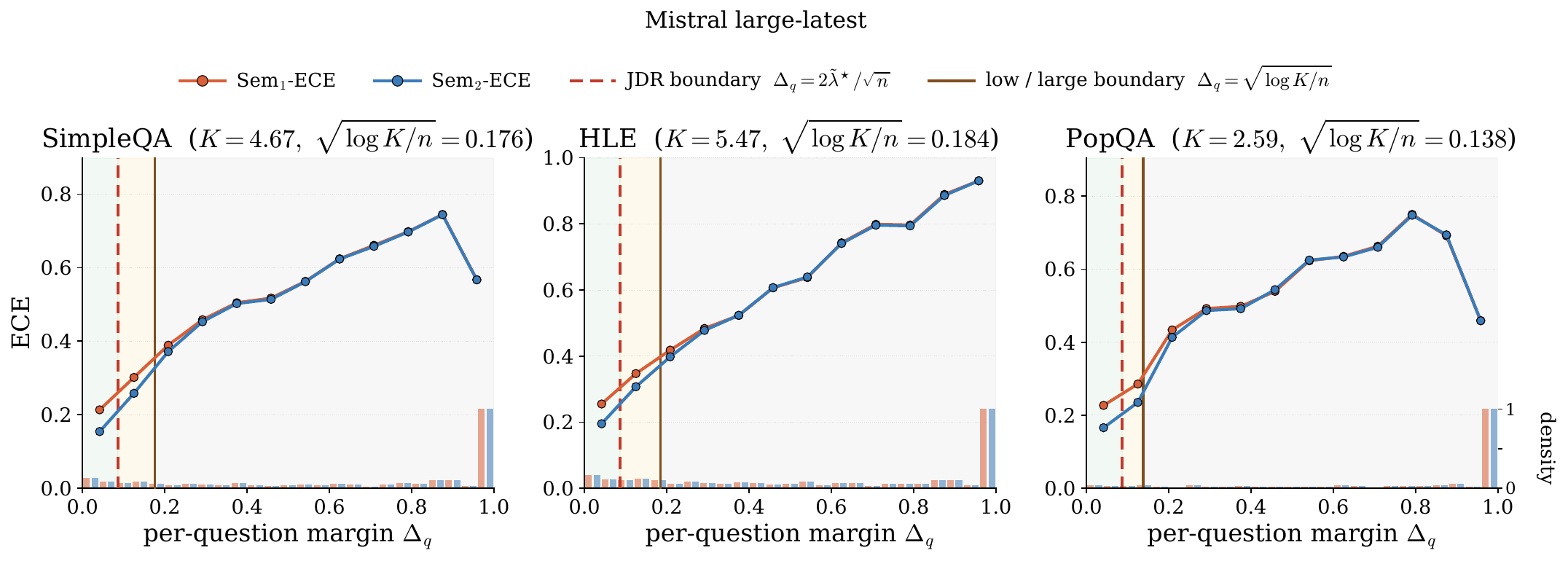}
\caption{
Margin-stratified ECE curves for Mistral on SimpleQA, HLE, and PopQA.
}
\label{fig:app-ece-margin-mistral}
\end{figure}
\subsection{Convergence Rate}
\begin{figure}[t]
  \centering
  \includegraphics[width=\linewidth]{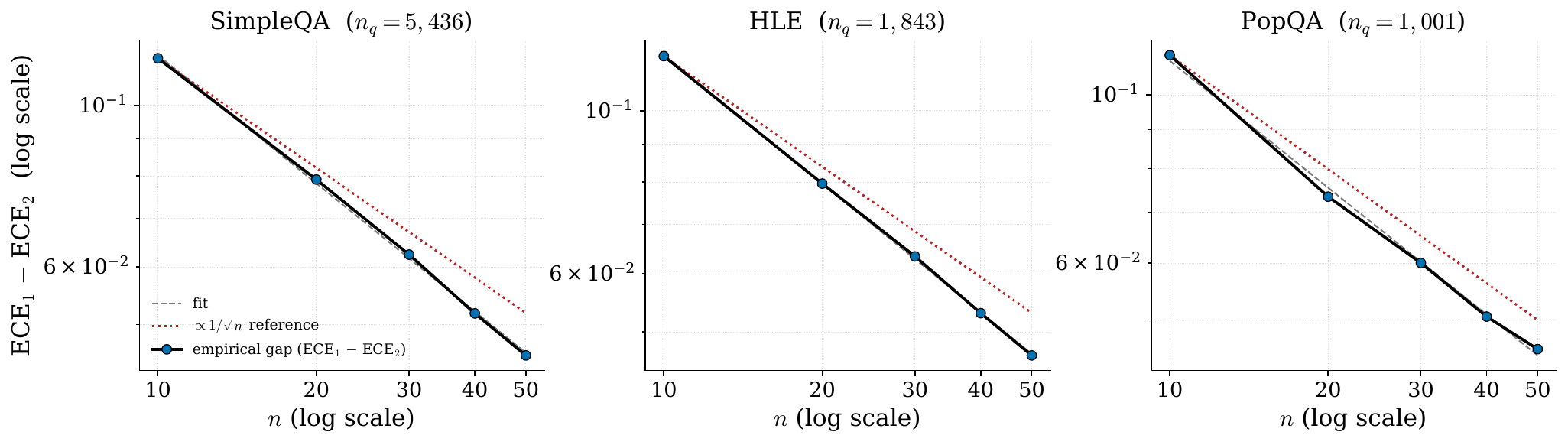}
  \caption{Direct ECE gap
  $\mathrm{Sem}_1\text{-}\mathrm{ECE} - \mathrm{Sem}_2\text{-}\mathrm{ECE}$
  on the low-margin sub-population
  $\{q : \Delta_q < \sqrt{\log K_q/n}\}$ on a log-log scale, threshold
  re-evaluated at each $n$. Fitted slopes $-0.58$, $-0.58$, $-0.56$
  are within $0.08$ of \Cref{thm:gap}'s prediction $-0.50$.}
  \label{fig:rate}
\end{figure}
\subsection{Reliability diagrams}
\begin{figure}[p]
\centering
\includegraphics[width=0.95\linewidth]{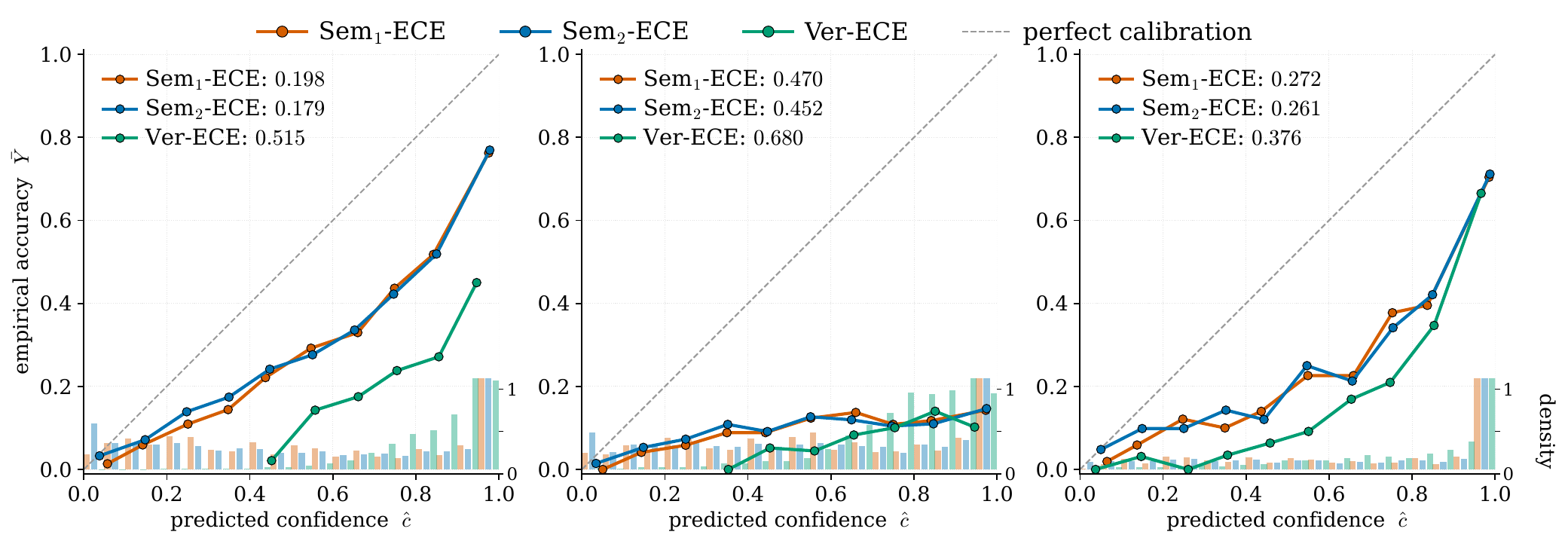}
\caption{
Reliability diagrams for OpenAI on SimpleQA, HLE, and PopQA.
}
\label{fig:app-reliability-openai}
\end{figure}

\begin{figure}[p]
\centering
\includegraphics[width=0.95\linewidth]{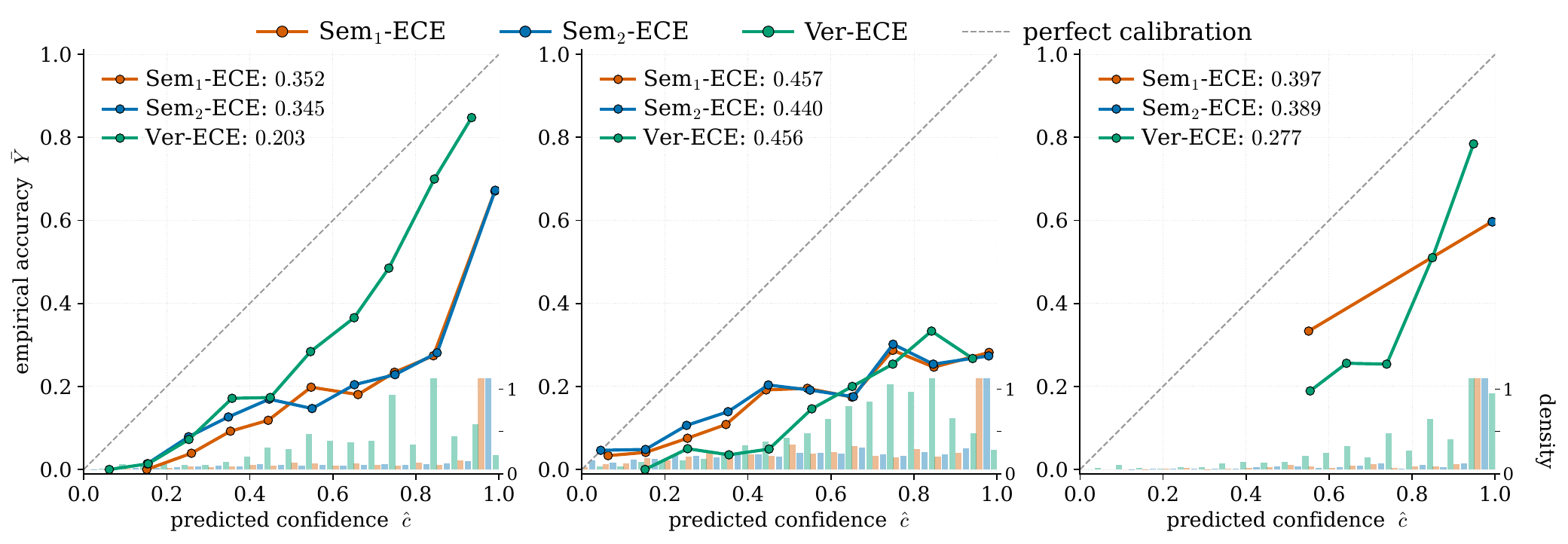}
\caption{
Reliability diagrams for Anthropic on SimpleQA, HLE, and PopQA.
}
\label{fig:app-reliability-anthropic}
\end{figure}

\begin{figure}[p]
\centering
\includegraphics[width=0.95\linewidth]{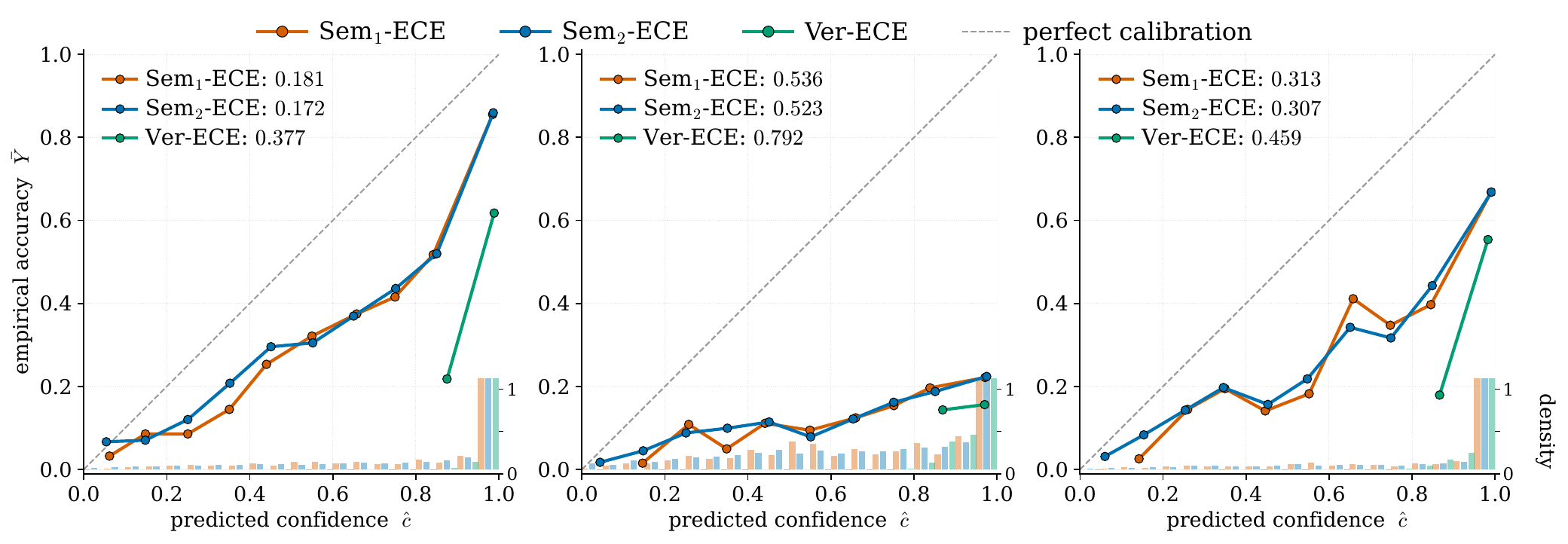}
\caption{
Reliability diagrams for Gemini on SimpleQA, HLE, and PopQA.
}
\label{fig:app-reliability-gemini}
\end{figure}

\begin{figure}[p]
\centering
\includegraphics[width=0.95\linewidth]{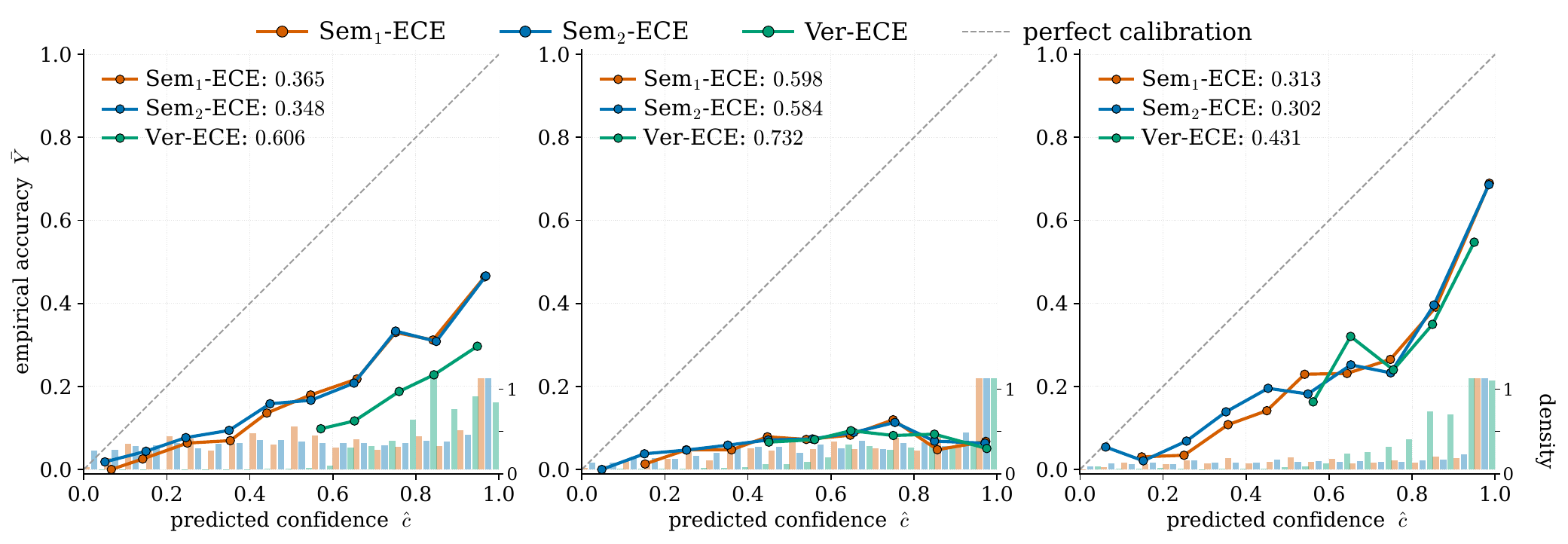}
\caption{
Reliability diagrams for xAI on SimpleQA, HLE, and PopQA.
}
\label{fig:app-reliability-xai}
\end{figure}

\begin{figure}[p]
\centering
\includegraphics[width=0.95\linewidth]{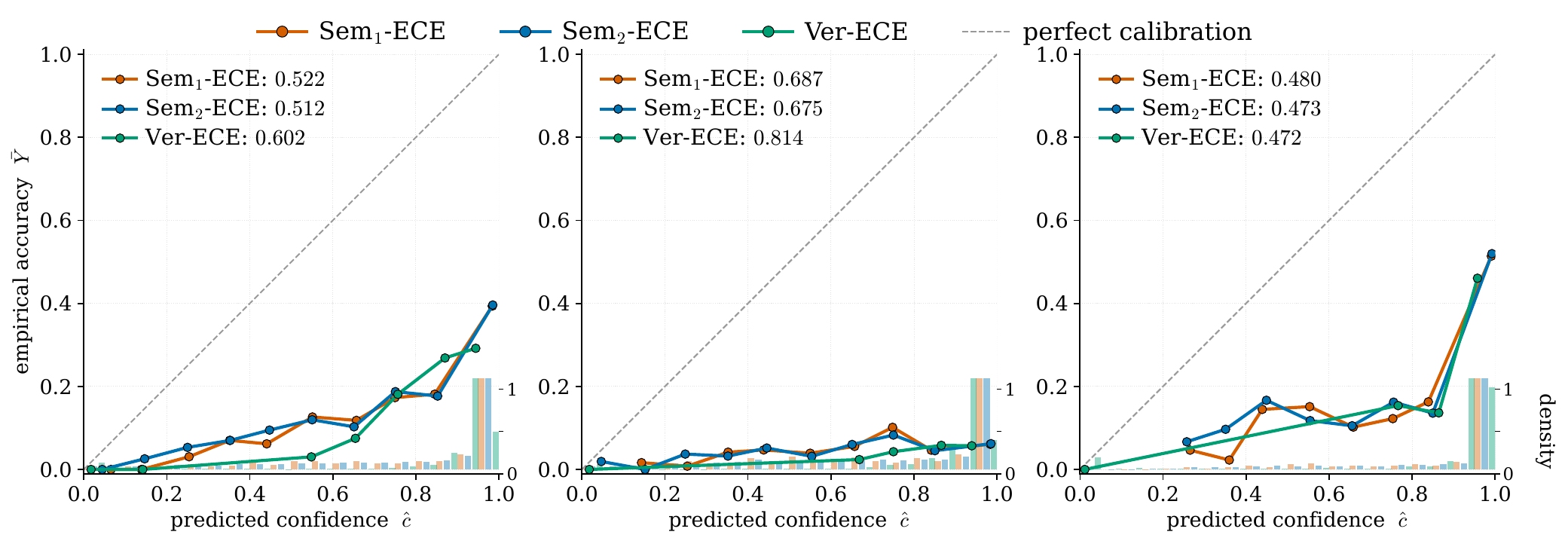}
\caption{
Reliability diagrams for Mistral on SimpleQA, HLE, and PopQA.
}
\label{fig:app-reliability-mistral}
\end{figure}

\section{Boundary alignment numerics}
\label{app:alignment}

\Cref{tab:alignment} reports the per-benchmark numerics that underlie
the leading-constant comparison in \Cref{sec:exp_sharp}. We measure
the empirical $\mathrm{Sem}_1$-ECE $-$ $\mathrm{Sem}_2$-ECE in a
$\pm 10\%$ window around each regime boundary on each pooled
benchmark, and compare to the leading-order prediction
$\varphi(\tilde m^\star)/\sqrt n$ from \Cref{thm:gap} under the
$p_q \to 1$ convention. The JDR boundary
$\Delta_q = 2\tilde\lambda^\star / \sqrt n \approx 0.0865$ at $n = 50$
is universal in $K_q$; the low/large boundary
$\Delta_q = \sqrt{\log K_q / n}$ depends on per-benchmark $K_q$.

\begin{table}[h]
\centering
\caption{Boundary alignment between empirical
$\mathrm{Sem}_1$-ECE $-$ $\mathrm{Sem}_2$-ECE and the leading-order
prediction $\varphi(\tilde m^\star)/\sqrt n$ under $p_q \to 1$, at
the two regime boundaries on each pooled benchmark ($n = 50$).
$n_q$ is the number of questions in the $\pm 10\%$ window.}
\label{tab:alignment}
\setlength{\tabcolsep}{5pt}
\small
\begin{tabular}{llcccccc}
\toprule
Dataset & $K_q$ & Band & $\Delta_q$ at bdy. & $n_q$
        & gap$_{\mathrm{emp}}$ & gap$_{\mathrm{theory}}$
        & emp/theory \\
\midrule
\multirow{2}{*}{SimpleQA} & \multirow{2}{*}{$7.48$}
  & JDR       & $0.0865$ & $510$ & $0.0526$ & $0.0468$ & $1.12$ \\
& & low/large & $0.2006$ & $597$ & $0.0151$ & $0.0206$ & $0.73$ \\
\midrule
\multirow{2}{*}{HLE} & \multirow{2}{*}{$7.49$}
  & JDR       & $0.0865$ & $276$ & $0.0545$ & $0.0468$ & $1.17$ \\
& & low/large & $0.2007$ & $360$ & $0.0183$ & $0.0206$ & $0.89$ \\
\midrule
\multirow{2}{*}{PopQA} & \multirow{2}{*}{$6.05$}
  & JDR       & $0.0865$ & $146$ & $0.0527$ & $0.0468$ & $1.13$ \\
& & low/large & $0.1897$ & $210$ & $0.0183$ & $0.0229$ & $0.80$ \\
\bottomrule
\end{tabular}
\end{table}

The leading-order prediction recovers the empirical gap to within
$11$--$27\%$ on every benchmark with no fitted constants. The same
direction of error appears across all three benchmarks: the JDR
boundary over-shoots and the low/large boundary under-shoots. This
sign-consistent residual, together with the steeper-than-$-0.50$
log-log slope in \Cref{fig:n-sweep}(b), is consistent with a
subleading $O(1/n)$ Edgeworth correction that flips sign as
$\tilde\lambda$ moves from $\tilde\lambda^\star \approx 0.306$ to
$\sqrt{\log K_q}/2$. A refined alignment that replaces the $p_q \to 1$
convention with band-wise $\hat p_q$ averages does not improve the
fit, ruling out the convention as the source of the residual.

\section{Bootstrap details}
\label{app:bootstrap}

For each model--benchmark cell we compute paired percentile bootstrap
$95\%$ CIs on three statistics, with $B = 1000$ replicates resampled
at the per-question level (questions are the i.i.d. unit; pairing
between $\hat c_1$ and $\hat c_2$ is preserved within each
resample):
\begin{itemize}[topsep=2pt,itemsep=1pt,leftmargin=*]
\item $\Delta\mathbb{E}[\hat c_1 - \hat c_2]$, the mean
per-question confidence reduction (positive $\Rightarrow$
$\mathrm{Sem}_2$ lowers confidence below $\mathrm{Sem}_1$);
\item $\Delta\mathrm{ECE} := \mathrm{Sem}_1\text{-ECE} - \mathrm{Sem}_2\text{-ECE}$
on all questions in the cell;
\item $\Delta\mathrm{ECE}_{\mathrm{low}}$, the same gap restricted to
the low-margin sub-population $\{q : \Delta_q < 1/\sqrt n\}$
(equivalent to $\tilde m_q < 1$ under the $p_q \to 1$ convention,
$n = 50$).
\end{itemize}
PopQA cells use the $466$-question intersection across all five
providers to enable paired comparison; this restriction reduces the
PopQA per-cell $N$ relative to \Cref{tab:per_pair_ece}, which uses
each provider's full coverage.

\paragraph{Summary.}
$\Delta\mathbb{E}[\hat c_1 - \hat c_2]$ is significantly positive on
all $15$ pairs --- the per-question confidence reduction predicted by
\eqref{eq:bias-expansion} holds without exception. $\Delta\mathrm{ECE}$
is significantly positive on $11$ of $15$ pairs; the remaining four
are PopQA cells (OpenAI, Anthropic, Gemini, xAI) whose $466$-Q
intersection sample is small enough that the population-level ECE
gap, while consistently directionally positive, does not always
exclude zero. On the theory-relevant low-margin sub-population
$\{q : \Delta_q < 1/\sqrt n\}$, $\Delta\mathrm{ECE}_{\mathrm{low}}$
is significantly positive in $11$ of $14$ measurable cells (PopQA
Anthropic has only $n_q = 26$ low-margin questions, insufficient for
a stable bootstrap CI). Details shown in \Cref{tab:bootstrap_dE}, \Cref{tab:bootstrap_dECE}, and \Cref{tab:bootstrap_dECE_low}.

\begin{table}[h]
\centering
\caption{Per-question confidence reduction
$\Delta\mathbb{E}[\hat c_1 - \hat c_2]$ with paired percentile
bootstrap $95\%$ CI ($B = 1000$, resampled per question). Cells whose
CI excludes zero are bolded. PopQA uses the $466$-Q intersection.}
\label{tab:bootstrap_dE}
\setlength{\tabcolsep}{5pt}
\small
\begin{tabular}{l|lll}
\toprule
& OpenAI & Anthropic & Gemini \\
\midrule
SimpleQA  & \textbf{$0.0197\ [0.0187,0.0206]$}
          & \textbf{$0.0074\ [0.0067,0.0081]$}
          & \textbf{$0.0089\ [0.0082,0.0097]$} \\
HLE       & \textbf{$0.0175\ [0.0160,0.0188]$}
          & \textbf{$0.0172\ [0.0155,0.0189]$}
          & \textbf{$0.0135\ [0.0123,0.0148]$} \\
PopQA     & \textbf{$0.0095\ [0.0073,0.0119]$}
          & \textbf{$0.0038\ [0.0023,0.0056]$}
          & \textbf{$0.0059\ [0.0040,0.0079]$} \\
\midrule
& xAI & Mistral & \\
\midrule
SimpleQA  & \textbf{$0.0161\ [0.0152,0.0170]$}
          & \textbf{$0.0101\ [0.0093,0.0108]$} & \\
HLE       & \textbf{$0.0144\ [0.0130,0.0157]$}
          & \textbf{$0.0119\ [0.0106,0.0130]$} & \\
PopQA     & \textbf{$0.0075\ [0.0053,0.0099]$}
          & \textbf{$0.0048\ [0.0031,0.0069]$} & \\
\bottomrule
\end{tabular}
\end{table}

\begin{table}[h]
\centering
\caption{Population ECE gap $\Delta\mathrm{ECE} =
\mathrm{Sem}_1\text{-ECE} - \mathrm{Sem}_2\text{-ECE}$ with $95\%$
CI. Cells whose CI excludes zero are bolded.}
\label{tab:bootstrap_dECE}
\setlength{\tabcolsep}{5pt}
\small
\begin{tabular}{l|lll}
\toprule
& OpenAI & Anthropic & Gemini \\
\midrule
SimpleQA  & \textbf{$0.0197\ [0.0170,0.0205]$}
          & \textbf{$0.0074\ [0.0067,0.0081]$}
          & \textbf{$0.0084\ [0.0063,0.0095]$} \\
HLE       & \textbf{$0.0175\ [0.0160,0.0188]$}
          & \textbf{$0.0171\ [0.0127,0.0187]$}
          & \textbf{$0.0135\ [0.0120,0.0148]$} \\
PopQA     & $0.0077\ [-0.0034,0.0123]$
          & $0.0078\ [-0.0012,0.0121]$
          & $0.0058\ [-0.0032,0.0127]$ \\
\midrule
& xAI & Mistral & \\
\midrule
SimpleQA  & \textbf{$0.0161\ [0.0152,0.0170]$}
          & \textbf{$0.0101\ [0.0093,0.0108]$} & \\
HLE       & \textbf{$0.0144\ [0.0130,0.0157]$}
          & \textbf{$0.0119\ [0.0105,0.0130]$} & \\
PopQA     & $0.0065\ [-0.0020,0.0093]$
          & \textbf{$0.0048\ [0.0015,0.0069]$} & \\
\bottomrule
\end{tabular}
\end{table}

\begin{table}[h]
\centering
\caption{Low-margin ECE gap $\Delta\mathrm{ECE}_{\mathrm{low}}$ on
$\{q : \Delta_q < 1/\sqrt{50}\}$ with $95\%$ CI. Cells bolded if CI
excludes zero. n/a indicates insufficient sample size.}
\label{tab:bootstrap_dECE_low}
\setlength{\tabcolsep}{5pt}
\small
\begin{tabular}{l|lll}
\toprule
& OpenAI & Anthropic & Gemini \\
\midrule
SimpleQA  & \textbf{$0.0476\ [0.0394,0.0485]$}
          & \textbf{$0.0548\ [0.0432,0.0587]$}
          & \textbf{$0.0498\ [0.0317,0.0552]$} \\
HLE       & \textbf{$0.0477\ [0.0438,0.0497]$}
          & \textbf{$0.0559\ [0.0319,0.0605]$}
          & \textbf{$0.0550\ [0.0504,0.0572]$} \\
PopQA     & $0.0274\ [-0.0413,0.0616]$
          & n/a
          & $0.0380\ [-0.0506,0.1131]$ \\
\midrule
& xAI & Mistral & \\
\midrule
SimpleQA  & \textbf{$0.0495\ [0.0482,0.0509]$}
          & \textbf{$0.0552\ [0.0531,0.0572]$} & \\
HLE       & \textbf{$0.0563\ [0.0536,0.0587]$}
          & \textbf{$0.0548\ [0.0494,0.0576]$} & \\
PopQA     & $0.0410\ [-0.0152,0.0793]$
          & \textbf{$0.0585\ [0.0106,0.0695]$} & \\
\bottomrule
\end{tabular}
\end{table}

\paragraph{Interpretation.}
The per-question reduction
$\Delta\mathbb{E}[\hat c_1 - \hat c_2]$ is the most directly
empirically observable consequence of \eqref{eq:bias-expansion}, and
its CI excludes zero on every cell --- including all PopQA cells
despite their reduced effective sample size. The population-level
$\Delta\mathrm{ECE}$ has a less stable CI because binning
discretization dampens the signal; this is most visible on PopQA at
$N = 466$. Restricting to the low-margin sub-population recovers a
larger and more stable signal exactly where \Cref{thm:gap} predicts
the gap to be largest, with effect sizes ($\sim 5$ percentage points)
roughly $4 \times$ those on the full population, which is consistent with
the regime structure visible in \Cref{fig:ece-margin}.

\section{Pooled reliability analysis}
\label{app:reliability}

\Cref{fig:reliability} shows reliability diagrams pooled across
models on each benchmark. $\text{Sem}_2$ achieves the lowest pooled
ECE on every benchmark (SimpleQA $0.311$, HLE $0.542$, PopQA $0.334$,
versus $\text{Sem}_1$ $0.323/0.556/0.340$ and Ver
$0.458/0.690/0.382$). All three sources are over-confident at high
confidence on HLE---expected for an expert-level benchmark where
models are highly self-consistent yet factually wrong, so semantic
agreement (which Sem-ECE measures) and factual correctness diverge.
The remaining ECE in $\text{Sem}_2$ reflects this population gap
$|\mathbb{E}_q[c_q^\star] - \bar a|$, a property of the underlying
model rather than of the calibration estimator. Combining a debiased
agreement metric like $\text{Sem}_2$-ECE with a separate signal
targeting $\bar a$ directly is a natural direction for future work.

\section{Extended Related Work}
\label{app:related}

\textbf{Calibration evaluation.}
Calibration evaluation is well studied for probabilistic classifiers, where a
model predicts a probability distribution over a fixed label set and metrics
such as Brier score, reliability diagrams, and expected calibration error compare
confidence with empirical accuracy
\citep{brier1950verification,naeini2015obtaining,guo2017calibration}.
This setting naturally extends to multiple-choice QA, where confidence can be
computed from logits or normalized option probabilities. Open-ended QA is less
straightforward: the answer space is not fixed, logits may be unavailable, and
correctness is semantic rather than lexical. As a result, calibration evaluation
for open-ended QA also requires a correctness judgment for free-form answers,
typically through human annotation or a validated automatic judge.

\textbf{Confidence estimation for open-ended language models.}
Several approaches estimate confidence without relying on a fixed label space.
Verbalized confidence asks the model to state its own uncertainty in words or as
a probability
\citep{lin2022teaching,kadavath2022language,mielke2022reducing,tian2023just}.
This is flexible with respect to answer format, but depends on the model's
self-reporting behavior and can be inaccurate or over-confident. Sampling-based
methods instead query the model multiple times and use agreement across
generations as a confidence signal
\citep{wang2023selfconsistency,lyu2025calibrating}. Related semantic uncertainty
methods group generations by meaning and aggregate uncertainty over semantic
clusters
\citep{kuhn2023semantic,farquhar2024detecting}. Existing sampling-based
calibration methods often still rely on task-specific answer extraction, such as
fixed final-answer patterns or regular expressions, before agreement can be
computed.

\textbf{Improving calibration.}
A separate line of work aims to improve model calibration rather than evaluate
it. Post-hoc methods such as Platt scaling, temperature scaling, Dirichlet
calibration, and verified calibration adjust confidence scores on held-out
validation data without changing the model
\citep{platt1999probabilistic,guo2017calibration,kull2019beyond,
kumar2019verified}. For language models, calibration can also be improved or
elicited through prompting, auxiliary confidence estimation, or verbalized
uncertainty
\citep{lin2022teaching,kadavath2022language,tian2023just}. Fine-tuning-based
approaches modify the model or training objective to preserve or restore
calibration after alignment
\citep{xiao2025restoring}.

\clearpage
\newpage

\end{document}